\documentclass{article}

\usepackage[english]{babel}

\usepackage[letterpaper,top=2cm,bottom=2cm,left=3cm,right=3cm,marginparwidth=1.75cm]{geometry}

\usepackage{natbib, tabularx, amsthm, amssymb, authblk}
\usepackage[colorlinks=true, allcolors=blue]{hyperref}
\usepackage{csquotes, enumitem, amsmath, tikz, graphicx}
\usetikzlibrary{matrix,fit,backgrounds,tikzmark, arrows.meta, calc}
\newtheorem{lemma}{Lemma}[section]

\date{September 2, 2025}
\title{Understanding Reinforcement Learning for Model Training,\\ and future directions with GRAPE}
\author{Rohit Patel}
\affil[]{Meta Superintelligence Labs\footnote{Thanks to Kushdesh Prasad for extensive reviews and feedback. Additional thanks to Vikas Bahirwani and Karl Tuyls for review and feedback. This work was done independently without Meta’s involvement or resources.}}

\begin{document}
\maketitle

\begin{abstract}
This paper provides a self-contained, from-scratch, exposition of key algorithms for instruction tuning of models: Supervised fine-tuning (SFT), Rejection Sampling, REINFORCE, Trust Region Policy Optimization (TRPO), Proximal Policy Optimization (PPO), Group Relative Policy Optimization (GRPO), and Direct Preference Optimization (DPO). Explanations of these algorithms often assume prior knowledge, lack critical details, and/or are overly generalized and complex. Here, each method is discussed and developed step by step using simplified and explicit notation focused on LLMs, aiming to eliminate ambiguity and provide a clear and intuitive understanding of the concepts. By minimizing detours into the broader RL literature and connecting concepts to LLMs, we eliminate superfluous abstractions and reduce cognitive overhead. Following this exposition, we provide a literature review of new techniques and approaches beyond those detailed. Finally, new ideas for research and exploration in the form of GRAPE (Generalized Relative Advantage Policy Evolution) are presented.
\end{abstract}

Modern Large language models (LLMs) have demonstrated remarkable capabilities across a wide range of tasks. Much of the progress in improving generation quality, aligning model behavior with human preferences, and improving safety in natural language tasks can be attributed to reinforcement learning (RL) methods, particularly reinforcement learning from human feedback (RLHF) \citep{ouyang2022training}. However, majority of existing resources on the topic remain inaccessible to practitioners without substantial time commitment. This is because: a) Familiarity with the broader RL literature is assumed, drawing on concepts, notation, and language from optimal control theory and game theory which can obscure concrete mechanics when applied to LLMs, b) Crucial details are often missing, especially those linking theory to implementation, c) Notation is inconsistent across resources and same thing can mean different concepts.

Here, we cover key aspects of Reinforcement Learning for Model Training, calling it as such (RLMT). This paper does three things:
\begin{enumerate}
    \item In Sections~\ref{sec:instructiontuned}-\ref{sec:rejectionsampling}, we start by covering Instruction tuning, SFT and Rejection Sampling to lay foundations. In Section~\ref{sec:rl}, we detail five key RLMT algorithms, REINFORCE \citep{williams1992simple}, Trust Region Policy Optimization (TRPO) \citep{schulman2015trust} Proximal Policy Optimization (PPO) \citep{schulman2017proximal}, Group Relative Policy Optimization (GRPO) \citep{shao2024deepseekmath}, and Direct Preference Optimization (DPO) \citep{rafailov2023direct} using the following philosophy:
\begin{itemize}
    \item Build from First Principles: We assume little prior knowledge, building up to every formula and method step-by-step from the basics
    \item Focus on intuition: We explain the motivation behind the math. You'll understand why a technique is needed (e.g. clipping in PPO) and how it intuitively solves a problem
    \item Ground everything in LLMs: Abstract ideas are always connected to their practical meaning. For instance, what exactly is it that the KL divergence is measuring the distance between in an LLM
    \item Ensure Clarity and Rigor: We use consistent industry-standard notation and do not \emph{hand-wave} details, ensuring a thorough and clear explanation
\end{itemize}
Our hope is that this will make the content more accessible and foster greater research in RLMT. 
    \item In Section~\ref{sec:emergingapproaches} we provide a literature review of other RLMT approaches such as curriculum learning, RLAIF, Process supervision, Self-play and others. 
\item Finally, in Section~\ref{sec:newresearchideas} we present some new ideas to the reader for exploration, bundled together as a new approach GRAPE (Generalized Relative Advantage Policy Evolution).
\end{enumerate}

We will start with pre-trained model capable of predicting next token, and walk through everything we need to train it into an instruction-tuned model. We have previously provided an intuitive treatment of LLMs in \citep{patel2024understandingLLMs} covering the basics of neural networks, training, embeddings, tokenizers, self-attention, layer normalization, dropout, multi-head attention, positional embeddings and the GPT and transformer architectures. In this paper we will abstract away the details of the tokenizer and simply think of a token as a word. As such, a pre-trained model works such that if you give it some text, it will try to predict the most likely token (i.e. word, for our purposes) as continuation of this text. Examples of pre-trained models are \citet{radford2018improving}, \citet{radford2019language}, \citet{touvron2023llama}.

While this might seem like glorified auto-complete, pre-trained models can in fact do remarkable things. For example, if I were to try to complete the sentence \enquote{The capital of Germany is}, the model will complete it with Berlin. Moreover you can recursively feed the text back to the model to get it to write longer text and articles as shown in Figure~\ref{fig:pretrain_recursive_feed}. For example, one might start giving the model the following text:
\enquote{This article explores Florida's alligators, covering their habitat, physical characteristics, diet, reproduction, interaction with humans, and fascinating facts.}
Continuing this text will cause a good pre-trained model to complete the entire article on Florida's alligators along with the aspects mentioned in the seed text.
\begin{figure}[h]
    \centering
    \includegraphics[trim=3cm 9cm 4cm 8cm,clip, width=\textwidth]{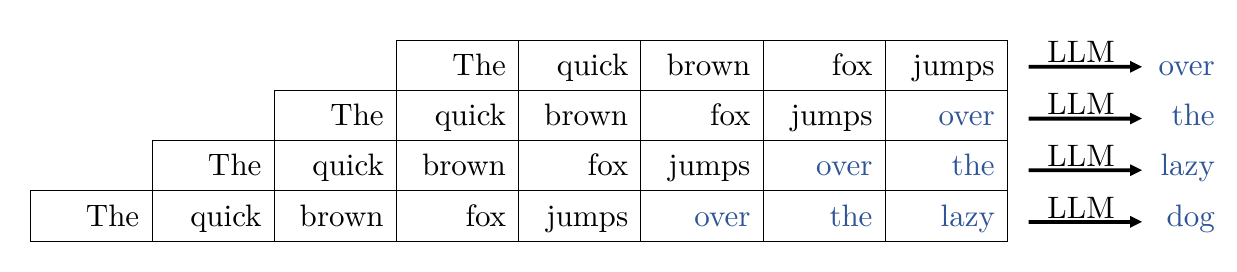}
    \caption{Recursively feeding generated tokens to LLMs for longer text generation}
    \label{fig:pretrain_recursive_feed}
\end{figure}

At this point, it is conceivable that anything you may want the model to do can be reframed as a completion objective, e.g. \enquote{This is a 1000 word essay on...} instead of \enquote{Write me a 1000 word essay on...}. So why do we need instruction tuned models?

\section{Instruction tuning models}\label{sec:instructiontuned}
Meaningfully using pre-trained models for useful purposes would require a change in behavior, which is difficult. Instruction tuned models, on the other hand, can be used in a natural question answering format, which can lead to ease of use and adoption. This is evident by the fact that LLMs really caught on when the first instruction tuned models were made available.

Being text completion models, pre-trained models are trying to predict the most likely next token based on the corpus of data that they are trained on. This means that pre-trained models may:
\begin{itemize}
    \item Respond to questions with additional questions since training data may contain lists of questions
    \item Not do a great job of following instructions
    \item Make up facts or names/places that sound plausible but don't exist
    \item Generate offensive or harmful content without much consideration
\end{itemize}
The fine-tuning process is often used to further train the models to address these shortcomings. The biggest visible change this leads to is that the models follow instructions in a question-answer format, and as such we will call these models instruction tuned models (IT models) e.g. \citet{ouyang2022training}, \citet{thoppilan2022lamda}, \citet{touvron2023llama2openfoundation}.

Now, how do we get the model to answer questions? The simplest thing we can think of is if the model was trained on data where questions were followed by answers then it would learn to do text completions in a way so as to complete questions with answers. In essence, nothing about the model changes except now it is conditioned more strongly to complete any text that's a question with an answer. This is the key to instruction tuning. 

Yet, this doesn't get us all the way to where we would like to be. For example, if you enter \enquote{alligator and crocodile} in any LLM today, it will respond with a description of the animals and similarities and differences. If these are merely text completion, the expectation would be to continue the sentence in some reasonable way, rather than to treat it as a question and try to return an answer. How is this achieved? One simple approach would be to structure the training data in the model so that questions are preceded with \enquote{User} and answers with \enquote{Model}. Something like this:
\begin{enumerate}
    \item \textit{User:} What is an antelope? \textit{Model:} An antelope is a type of mammal that belongs to the Bovidae family...
    \item \textit{User:} How many legs does a spider have? \textit{Model:} A spider has eight legs...
    \item \textit{User:} Who wrote Romeo and Juliet? \textit{Model:} Romeo and Juliet, the iconic tragic love story, was written by...
\end{enumerate}
By structuring the question and answer in the training data in this manner, whenever the model sees \enquote{User:} followed by some text followed by \enquote{Model:}, it will be conditioned to treat the text as a question and complete the whole sequence with something that looks like a complete response related to that text. This is similar in spirit to chat markup language suggested in \citet{ouyang2022training}.

This works relatively well, but there can be issues. For example, what if the string User: or Model: occurs naturally in the question or the answer, for example \enquote{How is the delimiter 'User:' typically used in model training?}, which is a valid question to ask an LLM. Can we improve this system? Let's go back to the basics of how text is fed into these neural networks. Remember neural networks can only ingest and output numbers. As such, all text is broken up into subword \enquote{tokens} and then fed into the network. Each of these tokens is represented by a vector which is its \enquote{embedding}. So, in essence, when you are feeding text to an LLM what you are really feeding is a sequence of vectors each representing a token.

One nice way to delineate a question from an answer would be to simply mint two new tokens (and corresponding embedding vectors) that are placed in the position of User: and Model:. This is similar to perhaps inventing a new character for delimiting files because all the existing characters may already exist somewhere. Let's call these newly minted tokens $<USER>$ and $<MODEL>$ for the sake of convenience. This does not mean that the strings $<USER>$ and $<MODEL>$ have special meaning to the model - we are just using them here on paper to denote the special tokens because we need some way of denoting them. In reality, if someone were to feed these strings to the model, they would have their own tokenization same as anything else e.g. $<THISSTRING>$ or $<THATSTRING>$ etc. These tokens simply represent delimiters (we could have called them $<DELIM1>$ and $<DELIM2>$ and it would make absolutely no difference anywhere except in this discussion) in this work. With this new scheme, an embedding can be trained, and you can now have the model always follow the special tokens with an answer-type completion. Modern instruct tuned models use some variation of this scheme which was suggested in \citet{askell2021general}.

Now we have a pretty good scheme for instruction tuning a model and making sure a text completion model starts to behave more and more like a question answering model. If we can design a good corpus of questions and answers, then doing additional training on the model using the very same next token prediction scheme will give us what we need. Moreover, this stage can now be used to enforce other values on the model that are of interest. For example we can:
\begin{itemize}
    \item Put questions in the training data with specific instructions and answers that follow those instructions so the model can be better at instruction following in its answers
    \item Have questions that can lead to harmful or offensive content in responses and then answers that are refusals so that the model will learn to refuse answers to potentially harmful questions (e.g. respond with \enquote{Sorry I cannot provide that information} when asked a question \enquote{How to make a dirty bomb?})
\end{itemize}
These are not the only things one can address during training. Moreover, the above-mentioned objectives of instruction following and responsible AI are often achieved using a mixture of approaches and not simply what is described here.

\section{Supervised fine tuning}\label{sec:sft}
Supervised fine-tuning is no different from pre-training. In both cases, you are doing the exact same thing, training the model to predict the next token (other objectives exist, such as masked language modeling used in \citet{devlin2019bertpretrainingdeepbidirectional}, but are not nearly as prevalent and outside the scope here). The key difference is during pre-training you are using a much larger corpus such as the common crawl whereas for SFT you have a much smaller and higher quality dataset. 

Let's actually write down the loss function for the next token prediction objective. Let's assume that we have an LLM with a vocabulary size of 32,000. Now this means that the output layer of the LLM will contain 32000 numbers. After softmax is applied, we get 32k numbers that are all between 0 and 1 and they all add up to 1 and as such this can be treated as the vector of probabilities for the 32k tokens. Let's name these 32k numbers with indexed variables such that the first number is $p_1$ the second one is $p_2$ and so on such that the $i^{th}$ number is $p_i$ going all the way to $p_{32000}$

Now, in the next token prediction case we have the training data. Let's say the sentence \enquote{The quick brown fox jumps over the lazy dog} is part of our training data. So what we're doing here is that if we give the model part of the text, say \enquote{The quick brown fox} then the model should predict \enquote{jumps}. What this means is that of all the 32k tokens in the vocabulary, the probability of the token \enquote{jumps} should be the highest\footnote{In practice, tokens are not the same as words. Words are often broken down into one or more tokens. For example the word jumps could be broken down into two tokens \enquote{jump} and \enquote{s}. For sake of comprehensibility, we will assume that our language models are using words as tokens.}. Let's say jumps is the $i^{th}$ token then what we want to do is maximize $p_i$.

Now, the value of each of the $p_i$ will be different depending on which previous tokens were fed. For example, if we now feed \enquote{The quick brown fox jumps} to the network we want to maximize another $p_j$ where $j^{th}$ token is \enquote{over}. We need some concise way of representing this. Basically what we want to capture is \textit{Probability of \enquote{jumps} in the condition that \enquote{The quick brown fox} was fed to the model}. And we want to be able to represent that easily for many words. Let's label all the tokens in the sentence $tok_1$, $tok_2$, $tok_3$ and so on.
\begin{align*}
\vec{text} = \underbrace{\frac{\text{The}}{tok_1}\frac{\text{quick}}{tok_2} \frac{\text{brown}}{tok_3}\frac{\text{fox}}{tok_4}}_{\vec{s}_5} \tikzmarknode{a_t}{\frac{\text{jumps}}{tok_5}}\frac{\text{over}}{tok_6}\frac{\text{the}}{tok_7}\frac{\text{lazy}}{tok_8}\frac{\text{dog}}{tok_9}    
\end{align*}
\begin{tikzpicture}[overlay, remember picture]
    \coordinate (shift_down) at (0,0); 
    \draw [<-, black] ($(a_t.south) + (shift_down)$) -- ($(a_t.south) + (0,-0.25) + (shift_down)$) node [below, align=left] {$_{a_5}$};
\end{tikzpicture}

 Let's use the vector $\Vec{text}$ to denote the vector of tokens which contains $T$ total tokens $tok_1,tok_2...tok_T$. It gets tiresome to write all the tokens repeatedly so let's use the notation $a_t$ to denote $tok_t$ and $\vec{s}_t$ to denote $tok_1, tok_2,...,tok_{t-1}$ which will make it easier for us to express things. Note $\vec{s}_t$ counts first $t-1$ tokens not first $t$ tokens. Let's use $\pi$ to denote the probability of an event. And let's use \enquote{$|$} to denote "conditional on". So we can write the above probability as:
\begin{align*}
    &\text{Probability of ``jumps'' given the text ``The quick brown fox''}\\
    \quad\text{or}\quad & \pi\left(\text{jumps} \mid \text{The quick brown fox}\right) \\
    \quad\text{or}\quad & \pi\left(tok_5 \mid tok_1,tok_2,tok_3,tok_4\right)\\
    \quad\text{or}\quad & \pi\left(a_5 \mid \vec{s}_5\right)
\end{align*}
Now if we want the model to learn the whole sentence, we would be predicting these probabilities recursively. Applying the chain rule of probability, we can multiply these probabilities to get the total probability that the model will generate the sequence. As such the objective is:
\begin{align*}
    &\text{Maximize: }\pi(tok_1)\times \pi(tok_2 | tok_1) \cdots \times \pi(tok_9 | tok_1, tok_2,\ldots, tok_8)\\
    \quad\text{or}\quad &     \text{Maximize: }\pi(a_1)\times \pi(a_2 | \vec{s}_2) \times \pi(a_3 | \vec{s}_3) \cdots \times \pi(a_9 | \vec{s}_9)
\end{align*}
These probability numbers are extremely small. For each one of these, the probabilities are one of 32k numbers that all add up to 1. So you are really multiplying and maximizing really tiny numbers. If all probabilities were equal, they would be $\frac{1}{32000} = 0.00003125$, and many of these probabilities will be much smaller. Multiplying them for even a small sequence such as this one could mean that the number you are trying to maximize is around $\left(\frac{1}{32000}\right)^9 = 2.84 \times 10^{-41}$ which has 40 zeros after the decimal and before the digits 284! That is a difficult number to work with. Fortunately, if we take a log, then $log(1/32000)= -10.37$, which seems to be a much more manageable number. Moreover, log has this nice property that $log(a\cdot b) = log(a)+log(b)$ and so if we take log of these expressions, instead of multiplying really really tiny numbers and getting tinier numbers still, we can simply add up a few reasonably sized numbers. Since log is a nice monotonic function maximizing $x$ and maximizing $log(x)$ means the same thing. So why don't we write our objective function in a more manageable way:
\begin{align*}
    & \text{Maximize: }\log(\pi(a_1))+ \log(\pi(a_2 | \vec{s}_2)) + \cdots + \log(\pi(a_9 | \vec{s}_9))\\
    \quad\text{or}\quad &\text{Maximize: } \sum_{t=1}^9 \log(\pi(a_t | \vec{s}_t))
\end{align*}
Now we have nice numbers to maximize. Keep in mind all the log numbers are going to be always negative since log of anything under 1 is a negative number and all probabilities are going to be under 1. So maximizing a negative number means you are minimizing its magnitude since -5 is greater than -10 on the real line. One nice and clean option is to rewrite the problem as a minimization problem with a negative sign instead. That way, you have positive numbers that you are trying to minimize. Moreover, not all sequences are length 9 so let's make this a slightly more general $T$ length of the sequence.
\begin{align*}
    \text{Minimize: }\quad Loss =& - \sum_{t=1}^T \log(\pi(a_t |\vec{s}_t))
\end{align*}
Here, what you are really doing is changing the parameters of the model while minimizing this loss, so these probabilities will change. If we want to be really clear we should add somewhere in the loss function the clarification that it is coming from a specific model. One way to do it is simply to put the model in the condition, such that what you are really saying is something like \textit{Probability of \enquote{jumps} conditional on \enquote{The quick brown fox} being fed to the model where the model is M}. Let's use $NLL(\vec{text},M)$ to denote the loss. This would make the loss look something like this:
\begin{align}
    \text{Find M to minimize: }\quad  NLL(\vec{text},M)  =&  - \sum_{t=1}^T \log\left(\pi_M(a_t\ |\ \vec{s}_t)\right) \label{eq:nll}
\end{align}
All the mathematical notation may make it look complicated, but as we know, it is rather quite simple. We are just trying to maximize the probability of the specific tokens from the model. This formulation is also called \enquote{negative log-likelihood} because the probability of the tokens under the model can also be thought of as the likelihood of the model given the tokens. The model that gives us better values of probabilities is more likely to be a better model. So you can say you are maximizing the likelihood function and trying to find the best model that gives you the highest value of the likelihood function. We can use gradient descent to minimize the loss once you have a loss function. Supervised fine-tuning does just that with the negative log-likelihood as the loss function. This loss is what pre-training uses as well, at a much larger scale. The NLL loss is also referred to as the \enquote{cross entropy loss} in deep learning literature, which we discuss in Section~\ref{sec:crossentropy}. \citet{fisher1922mathematical}, \citet{kullback1951information}, \citet{shannon1948mathematical}, \citet{hopfield1987learning} and \citet{bengio2003neural} lay some of these foundations.

One last thing to note here is that the loss written above is for a single example. Usually when you are training you have a lot of training examples to run supervised fine tuning on and you take average loss over all those examples. Let's say you denote the set of all training samples by $\mathcal{S}$ and suppose the set $\mathcal{S}$ has $S$ samples; then the loss looks something like this:
\begin{align}
    SFTLoss &=\text{Average of } NLL(\vec{text},M)\text{ for all }\vec{text}\text{ in } \mathcal{S} \nonumber\\
    &= \frac{1}{S}\sum_{i=1}^S NLL(\vec{text}_i,M) \label{eq:avgloss}
\end{align}
One thing to note is that $\vec{text}$ is the concatenation of the question and the answer for a particular sample. This also means that if we want the model to simply learn how to generate answers to questions, we could calculate the NLL starting from token $t$ where $t$ is the first token of the answer. This does not change anything we discussed above, other than the slight change to NLL definition. 

So we're ready to instruction tune the model using SFT, but the issue now is finding a corpus of question-answer pairs that covers a wide range of topics. The written world is full of books, articles, papers etc. but not nearly enough text exists in the form of Q\&A. One option is to have humans write question-answers, or source questions from somewhere (e.g. users asking questions online) and have humans write answers to supplement whatever Q\&A datasets one could find online. All of these are time consuming and resource intensive approaches. Nonetheless, SFT remains a critical first step in training models to be instruction tuned and getting them to a somewhat respectable place of generating answers.

\section{Rejection Sampling}\label{sec:rejectionsampling}
Our model can now answer questions, somewhat. What can we do to scale our training and make it better? One idea is that we can use the model itself to generate more question-answering data. If we could source questions from somewhere and have the model generate answers, we would have more question-answer pairs. This would substantially reduce the work of creating such a dataset for fine-tuning. The issue however is that since the model isn't well trained, it may not generate the best data. How can we get over this? 

Why don't we generate many responses for each prompt from the model and see if some of them are good responses to the prompts, effectively rejecting all the other responses. This should allow us to curate a model generated dataset that is of higher quality than the average model response. Moreover, it is a lot easier for people to select the best response from a group rather than type up a good response to a question. The process looks something like this:
\begin{enumerate}
    \item Start with a prompt
    \item Use the current model to generate G different completions (responses) for that prompt, by sampling with enough randomness (temperature) to get a variety.
    \item Rank the G responses.
    \item Select the top response (or top few responses) and discard the rest.
    \item Fine-tune the model on the prompt paired with that selected best response treating it as SFT data.
\end{enumerate}

We now have a way of using the model to generate the dataset for supervised fine-tuning. It is a model generated dataset but nonetheless we are going to use it to train the model using SFT. 

How can we further improve things? While it is easier for humans to find good responses from a sample, it is still a manual process. What if we had a model pick better responses? We could train a \enquote{preference model} that is capable of selecting the better of two responses when provided with two options. We would still need to have human labeled data for training this preference model, but after the initial set of human labeling work, we would be able to bootstrap an automatic process using the preference model. To train this preference model, we would thus need data that has the prompt and two generated responses where a human has selected a preferred response. In essence, you have two $\vec{text}_i, \vec{text}_j$ where each one is a combination of the prompt and one of the generated responses. Let $hp$ denote the actual human preference recorded in data such that $hp$ can take values zero and one, and if $hp=1$ then $\vec{text}_i$ is the human preferred text. The model takes the two $\vec{text}$ inputs and it gives as output the probability of each of the inputs being preferred by a human (i.e. it returns a single probability of the first input being preferred by human, let's call it $P(\vec{text}_i)$, since the other probability will simply be $P(\vec{text}_j) = 1-P(\vec{text}_i)$. Now what we want to do is maximize $P(\vec{text}_i)$ when $hp=1$ and maximize $P(\vec{text}_j)$ when $hp=0$. We can write the loss using the same trick above to convert probabilities to negative log-likelihood and turn it into a minimization problem:
\begin{align*}
    \text{Minimize}: -hp\cdot \log(P(\vec{text}_i)) - (1-hp) \cdot \log(1-P(\vec{text}_i))
\end{align*}
And so when human preference $hp=1$ then $P(\vec{text}_i)$ is being maximized and when $hp=0$ then $P(\vec{text}_j) = 1-P(\vec{text}_i)$ is being maximized. This is the same as before - and in the same way, to get the average loss we simply take an average over all prompts in the training dataset. This loss function is called \enquote{binary cross entropy loss}. The general concept of rejection sampling was suggested by \citet{von1951various}, whereas \citet{zelikman2022star} and \citet{yuan2023scaling} demonstrate usage in LLM training. We can also denote $\vec{text}_w$ as the winning text and simplify the expression:
\begin{align*}
    Loss = -\log(P(\vec{text}_w))
\end{align*}

Rejection sampling followed by supervised fine-tuning in a near-automated loop sounds like a great way to improve models. But it has certain issues and areas for improvement:
\begin{enumerate}
    \item \textbf{Discrete learning:} When we are rejection sampling, we are learning in discrete steps where we sample a lot of responses on a lot of prompts and then run a round of SFT from improvement. This means that multiple parameter updates are performed using the rejection-sampled data.
    \item \textbf{No Learning from Mistakes:} When we throw away the bad outputs, the model doesn’t learn why they were bad. It only gets signal from the one best answer that it was good. If the model repeatedly produces a certain kind of error in the rejected samples, this method doesn’t explicitly penalize that error. The feedback is purely positive (on the chosen answer) and not negative on the others. In essence, the model isn’t told what not to do, only what to do more of. This could limit the improvement or require many examples for the model to implicitly figure out the boundaries of bad responses.
    \item \textbf{High Computational Cost:} Generating many samples per prompt is computationally intensive, especially for large models. If we generate 10 candidates for each of 100k prompts, that’s 1 million model forward passes to sift out 100k best samples. This is much more work than a single pass per prompt. This can sometimes be mitigated by parallel generation or clever sampling, but it’s still a factor to consider.
    \item \textbf{Model Collapse and Bias:} This is the most important issue with rejection sampling. If we always pick the single highest-scoring answer according to a fixed criterion, we risk over-optimizing the model on that criterion. The model might start giving very narrow, optimized responses that score well but lack diversity or even coherence. For example, if the reward model or human annotator inadvertently prefers verbose answers, the model might converge to always giving overly long answers. In extreme cases, the model could exploit weaknesses in the scoring system, a phenomenon akin to reward hacking. Without any counterbalance, repeatedly fine-tuning on only top outputs can drive the model distribution to collapse around patterns that the scorer loves, even if those patterns are unnatural. We may be pushing the model into areas where the scoring model is not well calibrated, causing garbage outputs that the scorer mistakenly rates high. This issue is made worse by the fact that \enquote{No learning from mistakes} necessitates multiple rounds of rejection sampling followed by SFT increasing the chances of model collapse.
\end{enumerate}

\section{Reinforcement learning}\label{sec:rl}
How can we do better than rejection sampling? Let's try to fix the issues we listed above. Much of this section will deal with reinforcement learning methods, pioneered by \citet{brown1951iterative}, \citet{bellman1957dynamic}, \citet{barto1983neuronlike}, \citet{sutton1988learning}, \citet{watkins1989learning}, \citet{littman1994markov}, \citet{borgers1997learning}, \citet{hu2003nash}, and many others. However, we will discuss them in the context of LLM training. In the reinforcement learning literature, the set of rules that determine the value of $a_t$ given the current $\vec{s}_t$ is called a \enquote{policy}, and $a_t$ is referred to as \enquote{action} at time $t$ whereas $\vec{s}_t$ is referred to as the \enquote{state} at time $t$. In essence, a policy is something that provides a probability distribution over the set of actions (in our case, the set of actions is the model vocabulary, and the output of the final softmax layer is the probability distribution) given the current state $\vec{s}_t$ (in our case the current state is the string, i.e. the tokens, up to point $t-1$). You then sample the action $a_t$ from this probability distribution. For example, you could simply take the token with highest probability as $a_t$ (and that would be called greedy decoding). In our case, the model is the policy, and as such the terms model and policy are interchangeable for our purposes.

\subsection{REINFORCE}
One of the things we can do is that instead of treating the model generated samples as SFT data and running multiple parameter updates, we create new samples after each parameter update (or more practically a small number of updates, let's say less than five) thereby leading to a more continuous improvement in model and reducing the possibility of the model being overfit to a particular set of output. However, this would increase the compute burden many-fold. We already have to do a lot of generations to get a single training sample. So we first have to solve our issue of not using all the generated samples before we could do this.

How about using all the generated training samples in the loss by simply weighing them by their goodness? The current SFT formulation from Equation~\ref{eq:avgloss} simply averages the loss over all samples, effectively giving every rejection sampled example a weight of one. What if we had a function, let's call it a \enquote{reward function}, that would give us a score for each sample in terms of how good or bad it is. This way, we could use all generated samples during training and we would simply weigh the less good samples appropriately, or even negatively, in the loss function. It would look something like this:
\begin{align*}
    \text{Reward Weighted Loss} = \frac{1}{S}\sum_{i=1}^S R(text_i) NLL(\vec{text}_i,M)
\end{align*}
Where $R(text_i)$ is the reward from text sample $text_i$ and $NLL(\vec{text}_i,M)$ is as defined in Equation~\ref{eq:nll} . This is called the REINFORCE algorithm and is due to \citet{williams1992simple}. In practice, this loss function can lead to high variance in the gradients. To mitigate that, a baseline is subtracted from the reward. The baseline can be anything, but what is commonly used is something that depends on the input text being fed to the model for next token prediction (i.e. $\vec{s}_t$). Let's use $V_M(\vec{s}_t)$ to denote the baseline at time $t$. As such, we cannot separate the reward function from the sum in Equation~\ref{eq:nll} and we will need to expand the entire term. This is what it looks like:
\begin{align}\label{eq:reinfoceloss}
    &\text{REINFORCE Loss} = \hfill  \nonumber\\ &-\frac{1}{S}\sum_{i=1}^S \sum_{t=1}^T \left( R(\vec{text_i}) -V_M(\vec{s}_{it}) \right)\cdot \log(\pi_M(a_{it}\ |\ \vec{s}_{it}))
\end{align}
Where $\vec{text}_i$ is the $i^{th}$ sample text in the dataset and $a_{it}$ denotes the $t^{th}$ token of this text $i$ and $\vec{s}_{it}$ denotes the sequence of tokens of this text $i$ from 1 to $t-1$ (and as such, a whole $\vec{text}$ of length $T$ can also be denoted by $\vec{s}_{T+1}$). The baseline function $V_M(\vec{s}_t)$ depends on the text tokens we have up to $t$, and the model $M$. Finally, $\pi_M(a_t\ |\ \vec{s}_t)$ is the probability of the model $M$ giving the token $a_t$ as the next token when the sequence of tokens $\vec{s}_t$ (which is of length $t-1$ tokens) is input into the model, as discussed before.

Recall that each of the training sample $\vec{text_i}$ is essentially a combination of a prompt and a generated completion (in case of instruct-tuning, the completion would be a response). This whole sequence of prompt+response constitutes a single training example. If multiple responses were generated for one prompt, we would have multiple text sequences. One interesting thing to note here now is that once we have a reward function, we no longer necessarily need to do multiple generations for the same prompt. We can simply do one generation, improve the model, and do another generation with the improved model. Every generated response can be used for training and it makes the process a lot more efficient than rejection sampling. Another thing to note is that same as in the SFT case, since we want the model to only learn the answers to the questions and not the questions themselves, we start the inner sum of the term from a value of $t$ such that it corresponds to the first token of the string $\vec{text}_i$ where the answer starts, i.e. $t$ for each $i$ such that $\vec{text}_{it}$ is the first token of the answer following the question string. This can be done for all reinforcement learning algorithms, and as such we will not revisit this note again.

The process of taking actions, getting rewards, and learning from those rewards in a loop to improve future decisions is called reinforcement learning. REINFORCE is one of the simplest reinforcement learning algorithms, and it addresses the first three points that we made about rejection sampling. Let's finish our discussion of REINFORCE by understanding where these rewards come from - something that is relevant to many reinforcement learning algorithms.

\subsection{Value function}
Let's build some intuition into the baseline function $V_M(\vec{s}_t)$ that we just introduced. What would be a good definition for this function? Our very first formulation of REINFORCE loss was to weigh the probability of each $\vec{text}$ by the reward it achieves. However, we then introduced a baseline that was defined at token level, which means that we are now weighing the NLL of each token differently even within the same $\vec{text}$. Consider a high-reward text, meaning we want the model to have a high probability of returning those sequence of tokens, i.e., we want $\pi(a_t | \vec{s}_t)$ to be high for any t in that $\vec{text}$. What this implies is that $a_t$ is a high value next-token when you have $\vec{s}_t$ as the string so far. One thing we could do to make the model converge faster is to increase the loss when the model is far from returning the high-value tokens at $\vec{s}_t$, and vice-versa. How do we do that? 

Let's assume the reward $R(\vec{text}_t)$ is the reward we expect when selecting $a_t$ as the next token at $\vec{s}_t$ (since $\vec{s}_t$ and $a_t$ are coming from $\vec{text}$). Now if we could somehow get a sense of reward we would get otherwise, using model $M$ and continuing generation from $\vec{s}_t$, then we could establish that as the baseline. We know the probability of each token under $M$ (the output of softmax layer) and so if we knew the expected reward from selecting other tokens in the vocabulary as ${a_t}$ we could calculate the expected reward that the model $M$ would generate for us at $\vec{s}_t$. This is precisely what the value function $V_M(\vec{s}_t)$ is. Note that the value function is defined for each $\vec{s}_t$ and $a_t$ does not figure in the definition because we aggregate over all possible tokens $a_t$ that the model $M$ could have chosen. 

There is one issue here, though; to calculate $V_M(\vec{s}_t)$ we need to know the expected reward from selecting each possible token $a_t$ and we do not know that. We can only determine the reward using the reward function when we have completed generating the full answer. This implies that the calculation of $V_M(\vec{s}_t)$ requires us to play out all possible generations by this model until the end - every combination - record the probability of each one, then calculate the reward for each possible $\vec{text}$ generated, and then combine that with the probabilities to get the expected reward. This is cumbersome. Furthermore, the model $M$ is constantly changing as it is trained. One observation is that if you know $V_M(\vec{s}_{t+1})$ for all possible values of $\vec{s}_{t+1}$ (meaning all possible values of $a_t$ since $\vec{s}_t$ is fixed and $\vec{s}_{t+1}$ is simply a concatenation of strings $\vec{s}_t + a_t$) then you can write $V_M(\vec{s}_t)$ in terms of those terms. This is captured in equations by \citet{bellman1957dynamic} who also introduced the concept of value function. However, this does not help much in this situation. In practice, we often simply train another neural network to act as the value function and do not actually estimate the value function for a particular model. 

The model used as the value function is also called the \enquote{critic} model, with the main model $M$ also called the \enquote{actor} (another word for \enquote{policy}, in our case), and these reinforcement learning approaches called \enquote{actor-critic} methods.

\subsection{Advantage function}
The function above that we used to weigh the different log probabilities is called the advantage function. If we denote the advantage function by $A_M(a_t,\vec{s}_t)$ then in the above example we have:
\begin{align}\label{eq:advantagefunction}
    A_M(a_t,\vec{s}_t) & = R(\vec{text}) - V_M(\vec{s}_t)
\end{align}
Let's build a little intuition behind this beyond \enquote{we subtract a baseline to reduce variance}. The advantage function is defined at the token level, meaning that for every token $tok_t$ in $\vec{text}$ we have a defined value (remember both $a_t$ and $tok_t$ refer to the $t^{th}$ token in $\vec{text}$). This is different from our rewards, which are defined at the level of the entire text. What the advantage function is trying to capture is the \enquote{advantage} in selecting this particular token as $a_t$ over selecting whatever token the model $M$ might otherwise select. In a way, it gives you the additional reward that can by achieved by the model generating that specific $a_t$ vs following its own probability distribution - a fact that follows from our discussion of the value function above. The concept of advantage updating was first proposed by \citet{Baird1993}.

This is not the only form of advantage function that is used. In practice, a common advantage function used is GAE (Generalized Advantage Estimator), which we will not cover here. However, the function described above is a special case of GAE with two of its parameters (the discount factor, and the crediting parameter) set to one. Advantage estimators are an active area of research.

\subsection{Reward model}\label{sec:rewardmodel}
So where does the reward $R(\vec{text})$ come from? Typically, a model is trained to predict a score for each of the generations. This model can be trained using the same human preference data that was used to train the preference model for rejection sampling. The idea of training a reward model from human preferences was proposed by \citet{furnkranz2012preference} and applied by \citet{christiano2017deep} in the context of deep reinforcement learning. Let's take a look at how this can be done.

What we are trying to model is a reward function that essentially returns how good a response/text is, $R(\vec{text})$. The reward here represents how good humans find a particular response and, as such, when ranking, the person will choose the response with higher reward. However, what we have are human-ranked preferences on pairs of responses and not the reward itself. So we have to train a reward function that is consistent with these observed preferences. To train the model we need a loss function that is differentiable. How do we make a differentiable loss function from binary human preferences?

Interestingly, we have encountered this before. When doing next word prediction, we also only have the actual word choice available. This situation is similar except that the choice is made out of two responses instead of choosing a token from the whole token vocabulary. We can apply the same trick here, which is to treat the choice as a sample from output of probabilities that are differentiable. But we have another issue; here the rewards are scalar, not probabilities.

The output of the last layer of the language model are also scalar numbers. We can apply the same softmax here to convert reward scalars to probabilities. This would be a lot simpler since we only have two choices. Let's say $P(\vec{text_i})$ denotes that $\vec{text_i}$ was the chosen response between $\vec{text_i}$  and $\vec{text_j}$ then:

\begin{align*}
    P(\vec{text_i}) =  \frac{\exp{(R(\vec{text_i}))}}{\exp{(R(\vec{text_i}))}+\exp{(R(\vec{text_j}))}}
\end{align*}

Another way to approach this is to say that we know the text with greater reward should be preferred. So why don't we just take the difference of the reward function i.e. $\exp{(R(\vec{text_i}))}-\exp{(R(\vec{text_j}))}$ and if it positive then $i$ is preferred and if it is negative then $j$ is. Now, the difference in the reward is a scalar so we must convert this to a probability. This is also a situation we have seen before, and discussed the intuition behind. We would use the sigmoid function here to get the probability.
\begin{align*}
    P(\vec{text_i}) =  \frac{1}{1 + \exp{(-(R(\vec{text_i})-R(\vec{text_j})))}}
\end{align*}
The two approaches are the same and those two probabilities are mathematically identical if you work them out. Now, what do we normally do when we have a probability and we want to use it as a loss function? We take the negative log likelihood. Moreover, for this particular case when we have two choices we already wrote down the loss above which is the binary cross entropy loss: 
\begin{align}
\text{Reward function loss} &=  -hp\cdot \log(P(\vec{text_i})) - (1-hp) \cdot \log(P(\vec{text_j}))
\end{align}
Note that $P(\vec{text_j}) = 1 - P(\vec{text_i})$. This loss is very similar to the loss of the preference model. But we have a problem! While the preference model returned the probability directly, the reward function returns a single scalar. This loss requires values of both $R(\vec{text_i}), R(\vec{text_j})$ to calculate $P(\vec{text_i})$.

What we need to do here to calculate the loss is run two separate forward passes since unlike the preference model which returns the probability, here we need to calculate it from the rewards. In practice, this is not a huge deal as models are trained in batches anyway and loss is aggregated over examples in the batch before doing backward pass. In this case, if we put the $i,j$ pair in the same batch the loss is easily calculated from the outputs generated and backward pass can be done as usual. Now, if we assign $\vec{text}_i$ as the winning text (let's call it $\vec{text}_w$ and $\vec{text}_j$ as the losing text (say $\vec{text}_l$) then we can see that in equation above $hp=1$ and loss can be written as:
\begin{align}\label{eq:rewardfunctionloss}
\text{Reward function loss} &=  -\log(P(\vec{text_w}))\nonumber\\
&= - \log\left(\sigma(R(\vec{text_w})-R(\vec{text_l}))\right) 
\end{align}
Where $\sigma$ is the sigmoid function.

\subsection{Trust Region Policy Optimization}\label{sec:trpo}
Let's say we are iterating on model $M$ to reduce the loss, and let's say $M_0$ is the model at the beginning of the iteration and $M_1$ is the current model being iterated upon. One of the things that is happening here is that we are drawing samples from $M_0$. However, $M_1$ is the model being updated. This means that our sampling distribution does not match the actual distribution of the data. Is this a cause for concern and do we need to worry about it? In Section~\ref{sec:importancesampling} we provide an example of trying to guess the average height at a basketball event. The event contains 5\% pro basketball players and 95\% non-players. If we sample more people from pro players and get the average height of that sample that would be a biased estimate of overall height since pro basketball players are taller. We then discuss how to address this using importance sampling. Here, we can use importance sampling to get a better estimator. Moreover, one of the main issues in working with probabilities that we had was that they were really small, and the products quickly got out of hand. We resolved the issue by using NLL (negative log-likelihood). Importance sampling, on the other hand, is already a ratio of probabilities and as such we don't really need to take the log\footnote{In practice, probability ratios often involve calculating log values of numerator and denominator then taking a difference and exponentiating the result to avoid precision problems} to make it manageable. Applying the importance sampling adjustment in our case from Equation~\ref{eq:importancesampling} to Equation~\ref{eq:reinfoceloss} and removing the logs, we get the loss:
\begin{align}\label{eq:ratioloss}
    \text{Loss} &= -\frac{1}{S}\sum_{i=1}^S \sum_{t=1}^T \frac{\pi_1(a_{it}\ |\ \vec{s}_{it})}{\pi_0(a_{it}\ |\ \vec{s}_{it})}  A_1(a_{it},\vec{s}_{it}) 
\end{align}
Here we abuse the notation a little and use $\pi_1$ to denote the probability from $M_1$ instead of writing $\pi_{M_1}$, we will use this shorthand going forward.

Now, we have not yet addressed the most critical issue that we discussed with rejection sampling, that of model collapse. Repeatedly fine-tuning on the output from the model itself can cause the model to collapse. Why would this happen?

Firstly, there is no possible way for the training sample set $\mathcal{S}$ to contain all possible questions. If we make large updates to model parameters, it is entirely possible that while the model improves on the training sample, it gets worse overall. Doing this repeatedly can continue to make the model worse. This is the classic issue of overfitting.

Secondly, consider a situation where the reward model is trained on pairs of responses where it learns to detect certain features/characters of responses that are preferable to humans (e.g. better formatting) and it has only ever been presented with high-quality responses, having never seen bad quality content. Because the reward model has not been exposed to quality or coherence as key factors, it will focus on formatting to distinguish preferred responses. The model can give high reward to well-formatted low-quality incoherent responses, causing the main model to go into a destructive spiral of prioritizing incoherent well-formatted responses to maximize reward. This is a hyperbolic example using formatting, but in practice model collapse can happen in similar fashion in more subtle ways.

So what can we do here? What if we somehow restricted the optimization so that the new model did not go very far from the original model? The final output of these models are the probabilities of various tokens, so the most straightforward approach would be to restrict the output probabilities from changing too much from the original. However, we do not have control over the output probabilities directly during the training process, since training involves updating the model parameters. The parameter updates, in turn, are under our control by affecting the gradients or adjusting the loss function.  Thankfully, our loss function is written in such a way that it uses these output probabilities. So, one simple approach here would be to simply modify the loss function so that the probabilities do not go too far away from the original. In Section~\ref{sec:kldivergence} we looked at a measure of distance called KL divergence, commonly denoted by $D_{KL}(p\|q)$ for two distributions $p$ and $q$. So we at least know a way of measuring the distance between probability distributions. Let's formalize what we mean by \enquote{these distributions} here. Effectively, what we want is that for any $\vec{s}_t$ the distribution over next token (of possible values of $a_t$) does not end up being too far under the two models $M_0$ and $M_1$. Let us use $\pi_M(\cdot|\vec{s}_t)$ to denote this probability distribution so that the probability of $a_t$ under this distribution is $\pi_M(a_t | \vec{s}_t)$. Then KL divergence in our case turns out to be:
\begin{align}\label{eq:kldivergence}
    D_{KL}\left(\pi_0(\cdot|\vec{s}_t) \| \pi_1 (\cdot|\vec{s}_t)\right) &= \sum_{a_t} \pi_0(a_t | \vec{s}_t)\log\left(\frac{\pi_0(a_t | \vec{s}_t)}{\pi_1(a_t | \vec{s}_t)}\right)
\end{align}
Here, the sum is over all tokens in the model vocabulary, and $\pi_0(\cdot|\vec{s}_t)$ is essentially the last layer after softmax generated by the model $M_0$ when tokens $\vec{s}_t$ are the model inputs. Now we have a concrete way to measure the distance between the distributions generated by the two models $M_0$ and $M_1$. How can we use this? Why don't we simply add this term as a penalty to the loss function which will cause the optimization to penalize movements far away from the starting distribution. Since we don't know if the scale of our loss function matches that of the calculated KL penalty, we will need a scaling factor here. The loss can look something like this:
\begin{align}\label{eq:trpoloss}
\text{TRPO Loss} &= -\frac{1}{S}\sum_{i=1}^S \sum_{t=1}^T \frac{\pi_1(a_{it}|\vec{s}_{it})}{\pi_0(a_{it}|\vec{s}_{it})}  A_1(a_{it},\vec{s}_{it})+ \beta\cdot D_{KL}\left(\pi_0(\cdot|\vec{s}_t) \| \pi_1 (\cdot|\vec{s}_t)\right)
\end{align}
Let's recap what each of the the terms mean:
\begin{align*}
    M_0 = &\ \text{The original model to be updated}\\
    M_1 = &\ \text{The model being iterated on}\\
    S = &\ \text{Total number of text samples in the training set}\\
    T = &\ \text{The number of tokens in a text sample}\\
    \vec{text}_i = &\ \text{The $i^{th}$ text sample generated by $M_0$ now being used for RL}\\
    \vec{s}_{it} = &\ \text{The first $t$-1 tokens of $\vec{text}_i$}\\
    a_{it} = &\ \text{The $t^{th}$ token of $\vec{text}_i$}\\
    \pi_M(\cdot|\vec{s}_{t}) = &\ \text{Probability distribution over the vocabulary for the next token $a_t$ given}\\
    &\ \ \text{$\vec{s}_t$ as input string into model $M$}\\
    \pi_M(a_t|\vec{s}_{t}) = &\ \text{Probability token $a_t$ being the next token under distribution $\pi_M(\cdot|\vec{s}_{t})$ }\\
    A_M(a_t,\vec{s}_t) = &\ \text{Advantage of playing $a_t$ in state $\vec{s}_t$ instead of the default action in $M$}\\
     D_{KL} = &\ \text{KL divergence as discussed above}\\
     \beta = &\ \text{The scaling factor for KL divergence penalty term}
\end{align*}
This is the TRPO (Trust Region Policy Optimization) algorithm as described by \citet{schulman2015trust}. Much of the crucial theoretical groundwork for this was laid in \citet{kakade2002approximately} and \citet{kakade2001natural}. However, in the TRPO paper the authors observe that it is hard to balance the scaling factor $\beta$ to get step sizes that are not too small. One way to take larger step sizes more robustly is to use a constraint on KL divergence instead of putting it in the loss. This makes the TRPO optimization problem to be:

\begin{align}
\label{eq:trpoopt}
\text{minimize over $M_1$} &: -\frac{1}{S}\sum_{i=1}^S \sum_{t=1}^T \frac{\pi_1(a_{it}|\vec{s}_{it})}{\pi_0(a_{it}|\vec{s}_{it})}  A_1(a_{it},\vec{s}_{it})\\
\text{subject to constraint} &: D_{KL}\left(\pi_0(\cdot|\vec{s}_{it}) \| \pi_1(\cdot|\vec{s}_{it})\right) \le \delta &\text{$\forall i,t$, for some $\delta$}\nonumber
\end{align}
The authors suggest a further relaxation of the constraints to only constrain the average value of $D_{KL}$ over all tokens and samples, rather than for each token in each sample.

\subsection{Proximal Policy Optimization}\label{sec:ppo}
One of the things that TRPO achieves is to limit how much the output probability distribution for $a_t$ given an $\vec{s}_t$ diverge from the original model. In their paper titled Proximal Policy Optimization, \citet{schulman2017proximal} propose another way to introduce this constraint. Let's go back to Equation~\ref{eq:ratioloss} and see how we can do this differently. Here's the equation written for just one $\vec{text}$:
\begin{align*}
    \text{Loss} &= - \sum_{t=1}^T \frac{\pi_1(a_{t}|\vec{s}_{t})}{\pi_0(a_{t}|\vec{s}_{t})}  A_1(a_{t},\vec{s}_{t}) 
\end{align*}
As we iterate over $M_1$ we are really just trying to make sure that the probabilities of $a_t$ don't drift too far away. The probabilities would drift away in the first place as a result of our loss minimization strategy. And one insight here is that if we capped the loss function every time the two probabilities were further apart than a certain value (let us say $\epsilon$), then there would be no gain that the optimization process would achieve from moving the probabilities further apart. Since they are already written as a ratio, why don't we simply make sure that ratio is in the neighborhood of 1 by clipping the ratio in the loss function and taking max (or min inside the negative sign). This gives us the PPO Loss function:
\begin{align}\label{eq:ppoloss}
\text{PPO Loss} = - \frac{1}{S}\sum_{i=1}^S \sum_{t=1}^T \min & \left[ \frac{\pi_1(a_{it}|\vec{s}_{it})}{\pi_0(a_{it}|\vec{s}_{it})}  A_1(a_{it},\vec{s}_{it}) ,\right.\\
&  \left. \ \ \ \text{CLIP}\left( \frac{\pi_1(a_{it}|\vec{s}_{it})}{\pi_0(a_{it}|\vec{s}_{it})}, 1-\epsilon,1+\epsilon \right)A_1(a_{it},\vec{s}_{it}) \right]\nonumber
\end{align}
The CLIP function simply clips the provided value at both ends and, as such, CLIP$(v,x,y) = \min(\max(v,x),y)$. The PPO paper additionally presents a method involving an adaptive KL penalty coefficient, specifically the $\beta$ in Equation~\ref{eq:trpoloss}, which is modified following each minibatch optimization step. However, the authors note that clipping, described above, performs better in practice.
\subsection{Loss function used in practice}
The TRPO and PPO loss functions above are not exactly what is popularly used today. It is common to put the KL penalty term in the reward function itself (instead of the loss) while using the optimization algorithms above (e.g. \citet{ziegler2020finetuninglanguagemodelshuman}, \citet{stiennon2022learningsummarizehumanfeedback}, \citet{ouyang2022training}). This amounts to modifying the reward function to look something like this:
\begin{align}\label{eq:rewardwithpenalty}
    R'(a_t,\vec{s}_t, M_0, M_1) &= R(a_t,\vec{s}_t) - \beta \log\left(\frac{\pi_1(a_t | \vec{s}_t)}{\pi_0(a_t | \vec{s}_t)}\right)
\end{align}
Where $R'$ is the modified reward function. We can similarly write a modified Advantage function $A'$ which is simply defined by replacing the reward function $R$ in $A$ by the reward function $R'$.
\begin{align*}
    A'_M(a_t,\vec{s}_t, M_0, M_1) & = R'(a_t,\vec{s}_t, M_0, M_1) - V_M(\vec{s}_t)
\end{align*}
The idea is that putting the KL Divergence term in the reward function serves the same purpose as in Equation~\ref{eq:trpoloss} - that of keeping the updated model distribution closer to the original model distribution. There is a lot going on here:
\begin{enumerate}
    \item The reward function with which we had initially worked defined a reward for all of $\vec{text}$, such that $R(\vec{text})$ was the reward. Here, the arguments of the reward function contain $a_t$ and $\vec{s}_t$, implying that we have rewards at each token level.
    \item The penalty looks familiar to the definition of $D_{KL}$ but has two key differences:
    \begin{itemize}
        \item KL Penalty is a comparison of distributions, and as such you aggregate the term inside the logarithm over all values of $a_t$. Here it appears that we are simply taking a single term, the $a_t$ in the sample
        \item $M_0$ and $M_1$ appear to be flipped where in $D_{KL}$ we had $\pi_0(\cdot|\vec{s}_t)$ in the numerator and $\pi_1(\cdot|\vec{s}_t)$ in the denominator
    \end{itemize}
    \item Putting the loss in reward function is not the same thing as having the loss in optimization function, which is what TRPO describes
    \item It appears the reward is now dependent on the model, which seems odd
\end{enumerate}
Let's talk about each of these things below. 

First, let's talk about the fact that our reward function is trained on human preference rankings of full text sequences (question-answer pairs) and as such provides reward only for the full $\vec{text}$ string. When we pull the reward inside the summation over tokens term in Equation~\ref{eq:reinfoceloss}, the reward $R(\vec{text})$ can be interpreted as the reward for each token. In other words, setting a token-level reward for each token in the sequence $R(a_t, \vec{s}_t) = R(\vec{text})$ is the same thing as weighing the negative log likelihood of the entire sequence with the final reward. However, this formulation opens up the possibility of modifying the reward at the token level and doing more with it. That is what we are doing here.

Next, let's address the fact that the KL divergence term in Equation~\ref{eq:trpoloss} was $D_{KL}(\pi_0(\cdot|\vec{s}_t)\|\pi_1(\cdot|\vec{s}_t))$ whereas in Equation~\ref{eq:rewardwithpenalty} it is more similar to $D_{KL}(\pi_1(\cdot|\vec{s}_t)\|\pi_0(\cdot|\vec{s}_t))$. From discussion in Section~\ref{sec:kldivergence} we know that the two are not the same as KL divergence is not symmetric. As a side note, $D_{KL}(\pi_1(\cdot|\vec{s}_t)\|\pi_0(\cdot|\vec{s}_t))$ is also sometimes called reverse KL divergence (of $D_{KL}(\pi_0(\cdot|\vec{s}_t)\|\pi_1(\cdot|\vec{s}_t))$), but we will simply refer to it as $D_{KL}(\pi_1(\cdot|\vec{s}_t)\|\pi_0(\cdot|\vec{s}_t))$. Moreover, the $D_{KL}$ definition includes expectation over all possible values of $a_t$ (which would be the model's vocabulary, i.e. full token space) and that is clearly not being done in this modification.

The key is that when calculating $D_{KL}$ the sum (i.e. expectation) is taken over all possible $a_t$ under the first distribution in KL (i.e., the numerator. In the case of TRPO, that's $M_0$). However, our overall loss is defined by simply aggregating over all the samples and all tokens (the two outside sums). These samples are drawn from the latest iteration of the model (which is $M_1$ if we are iterating towards $M_2$, and so on) so if we could somehow use that probability term in $D_{KL}$ (i.e. use $\pi_1(a_t | \vec{s}_t)$ instead of $\pi_0(a_t | \vec{s}_t)$ as the probability outside log) we could vastly simplify our lives. Why? Because then we will not have to calculate the penalty term $D_{KL}$ for each token at all, we could simply place the log term in the sum, and the sums we are already taking will turn that into $D_{KL}$. The definition of $D_{KL}(\pi_1(\cdot|\vec{s}_t)\|\pi_0(\cdot|\vec{s}_t))$ provides exactly that. So, essentially by placing the term $\beta\cdot \log(\pi_1(a_t | \vec{s}_t) - \log(\pi_0(a_t | \vec{s}_t))$ in the reward function, and if everything else is linear, the term simply aggregates into an estimate of $D_{KL}(\pi_1(\cdot|\vec{s}_t)\|\pi_0(\cdot|\vec{s}_t))$ in the final optimization function as long as the advantage function is linear in the reward function (such as the one we saw above, but also the commonly used GAE advantage function which we have not discussed).

This feels hacky - we simply changed the KL Divergence term knowing that it is not symmetric. We discuss the impact of this asymmetry in Section~\ref{sec:kldivergence} but the key thing is that by placing $M_1$ in the numerator what we are ensuring is that $M_1$ does not assign probabilities to tokens that $M_0$ believes are not suitable, otherwise with the denominator being close to zero the term inside the log will explode. This has the impact of focusing $M_1$ in areas that $M_0$ has a high probability mass, which is what we want.

Lastly, let's talk about why putting the loss in the reward function works similarly to putting it in the optimization. We briefly alluded to this above, but if the advantage function is linear in reward function, then any term we add to the reward function can simply be separated out and aggregated independently. That is what happens here, and in-fact that is what allows us to simply put the log term in the reward function rather than calculating the full KL divergence for each token. The reward being dependent on the model being trained appears odd, but as we note the terms separate out nicely and it remains merely an implementation detail.

One key motivation of using the reward function modification for penalty here aside from the computational efficiency is that it allows practitioners to use the standard TRPO and PPO packages without modification, simply by specifying the updated reward function.

Now that we have discussed the reward function modification, what is the loss that is used in practice with PPO? Papers will generally use the modified reward with clipping. However, they are used slightly differently. Let's say $M_0$ is the original model and we draw samples from it and then do a round of updates to get $M_1$. Then we draw further samples from $M_1$ and do another round of updates to get $M_2$ and so on. The way the reward function modification is used is that the denominator always contains the original SFT model $M_0$ so as to keep the entire optimization process tethered. However, with clipping the ratio is often taken with the latest model $M_k$ from which the samples were drawn. One last thing to note is that the advantage function commonly used for PPO is Generalized Advantage Estimator (GAE) introduced in \citet{schulman2016high}.

\subsection{Group Relative Policy Optimization}\label{sec:grpo}
PPO is a very popular algorithm for reinforcement learning of LLMs, but the commonly used definition of advantage function requires the need for training a reward model and a value model in order to estimate the advantage. Is there room for improvement? The whole idea of an advantage function is to advantage the probability of the winning $\vec{text}$. A value function provides a baseline for how well a model can do from any starting string $\vec{s}_t$ and commonly used advantage function such as the one we defined in Equation~\ref{eq:advantagefunction} or the GAE commonly used with PPO use the value function to calculate advantage. 

Let's go back to draw some inspiration from rejection sampling. What if we generated multiple responses to every question and instead of selecting only the most promising response for SFT, what if we calculated the advantage of each response in the group by simply taking the difference from the average reward in the group? This would allow us to use all the samples, and would not require a value function to estimate the advantage. This is the core idea behind GRPO which was introduced in \citet{shao2024deepseekmath} and gained popularity with the launch of Deepseek V3 \cite{deepseekai2025deepseekv3technicalreport}. Here, we will discuss a variation of GRPO (called Dr. GRPO) presented in \cite{seaailab2025drgrpo} which eliminates some of the biases in the original GRPO paper. Let's say we have $Q$ questions in the database and for each one we generate $G$ responses so that our total sample size is $S=QG$. Let $\vec{text}_{qg}$ denote the full text with the question and the response concatenated for the $q^{th}$ question's $g^{th}$ response. Then the advantage function can be written as:
\begin{align}\label{eq:grpoadvantage}
    A(\vec{text}_{qg}) = R(\vec{text}_{qg}) - \frac{1}{G}\sum_{i=1}^G R(\vec{text}_{qi})
\end{align}
And the loss remains exactly the same PPO loss:
\begin{align}\label{eq:grpoloss}
\text{GRPO Loss} = - \frac{1}{QG}\sum_{q=1}^{Q}\sum_{g=1}^{G} \sum_{t=1}^T \min & \left[ \frac{\pi_1(a_{qgt}|\vec{s}_{qgt})}{\pi_0(a_{qgt}|\vec{s}_{qgt})} A(\vec{text}_{qg}) ,\right.\\
&  \left. \ \ \ \text{CLIP}\left( \frac{\pi_1(a_{qgt}|\vec{s}_{qgt})}{\pi_0(a_{qgt}|\vec{s}_{qgt})}, 1-\epsilon,1+\epsilon \right)A(\vec{text}_{qg}) \right]\nonumber
\end{align}
So in essence, GRPO is in-fact PPO but with a different advantage function than GAE which is commonly used in PPO. The key innovation of GRPO is not a new update mechanism, but the new advantage function to be used with PPO. Let's clarify a few things that look different and note some interesting facts about GRPO:
\begin{itemize}
    \item The dual sum on $Q$ and $G$ is the same thing as summing $i$ from 1 to $S$ where $S=QG$
    \item The index $it$ in $a_{it}$ was used to refer to the $t^{th}$ token in the $i^{th}$ completed text of the sample. Now sample has $QG$ data points and we are using $a_{qgt}$ to refer to $t^{th}$ token of the $g^{th}$ completion of the $q^{th}$ sample. Apart from using a different index to number the sample so we can track the groups, this fundamentally changes nothing. Same for $\vec{s}_{qgt}$
    \item The advantage function $A(a_t,\vec{s}_t,M)$ in Equation~\ref{eq:advantagefunction} was defined at token level.  The GRPO advantage function on the other hand is defined for the entire $\vec{text}$ sequence. This means we are using the same advantage function for each token in the sequence.
    \item The advantage function definition does not depend on the value function anymore, and depends only on the reward function. 
    \item The advantage function $A(a_t,\vec{s}_t,M)$ also depended on the model, since the value function depends on the model. The GRPO advantage function does not depend on the model being trained, but it does depend on the other generations in the group.
\end{itemize}

\subsection{Direct Preference Optimization}\label{sec:dpo}
How can we simplify things further? For PPO, we note the need for training a reward model and a value model in order to estimate the advantage function which allows us to calculate the loss. However, when we are training the preference model or the reward model, we are able to directly use human preferences to calculate a loss. The trick ultimately is to write a differentiable loss function that can help you update the model parameters in the right direction. 

For the reward model, we have a model that outputs reward and what we are really trying to do is nudge the model in the direction of increasing rewards for preferred $\vec{text}$ over the one less preferred. Our main model $M$ outputs probabilities, and we do want a similar dynamic here where we would like the model to increase probabilities of preferred $\vec{text}$. Following Equation~\ref{eq:nll}, let us continue working with negative log-likelihood of the responses - wanting to lower the NLL of a winning text (let us call it $NLL(\vec{text}_w)$) which is preferred over another losing text, $NLL(\vec{text}_l)$, in human preference. 

We can simply treat $NLL$ sort of negative of reward (since usually we want to maximize reward and minimize NLL) and pretend that our model generates reward and train it in the same way as a reward model. If follow identical reasoning to Section~\ref{sec:rewardmodel} and write the loss function below where $R(\vec{text})$ is simply replaced by $- NLL(\vec{text})$ we can use Equation~\ref{eq:rewardfunctionloss}
\begin{align*}
\text{Loss} &= -\log\left(\sigma\left(-(NLL(\vec{text}_w)-NLL(\vec{text}_l))\right)\right)\\
&= -\log\left(\sigma\left(\log(\pi_1(\vec{text}_w))-\log(\pi_1(\vec{text}_l)\right)\right)
\end{align*}
Where $\pi_M(\vec{text})$ is the probability of the full $\vec{text}$ being generated by model $M$. So, now we have a loss function written in terms of probabilities which our model does generate. We can use this to directly optimize the model for human preferences. However, in our eagerness, we have reintroduced the risk of model collapse that we had assiduously solved for, using clipping and the KL penalty. 

In their paper introducing Direct Preference Optimization, \citet{rafailov2023direct} show that modifying the loss above to replace $\pi_1(\vec{text})$ by the ratio $\frac{\pi_1(\vec{text})}{\pi_0(\vec{text})}$ along with a scaling factor $\beta$, is equivalent to having KL penalty in the reward. We will not reproduce the proof here. The final DPO loss can be written as:
\begin{align*}
\text{DPO Loss} &= -\log\left[\sigma\left\{\beta\log\left(\frac{\pi_1(\vec{text}_w)}{\pi_0(\vec{text}_w)}\right)-\beta\log\left(\frac{\pi_1(\vec{text}_l)}{\pi_0(\vec{text}_l)}\right)\right\}\right]
\end{align*}
And now we simply aggregate this loss over all samples in the data.

\section{Emerging approaches}\label{sec:emergingapproaches}
We covered some of the key approaches and algorithms used to train modern LLMs from pre-trained models to the wildly successful instruction tuned models in use today. Even with these successes, researchers are running into major roadblocks that are pushing the field in new directions. RLHF, for all its strengths, is slow and resource intensive because it requires a massive amount of high-quality preference data collected from people. This has made it hard to scale up. At the same time, methods like DPO work best when they have lots of direct comparisons (e.g., "response A is better than response B"). However, they are not as good for complex tasks, like multi-step math problems, where the feedback is often just a single \enquote{correct} or \enquote{incorrect} for the final answer. This makes it hard to know which specific step in the reasoning was good or bad. To solve these problems, research is moving forward in several key areas. To make training more scalable, researchers are exploring Reinforcement Learning from AI Feedback (RLAIF), which uses another AI to provide feedback instead of humans. To handle complex reasoning, the focus is shifting from rewarding just the final outcome to supervising the model's step-by-step thinking process. Finally, as models become more powerful, new training methods are being developed using game theory and self-play, like having AI agents debate each other or play against themselves to generate useful training signals. Let's briefly summarize these new research directions. 

\subsection{Curriculum learning}
The idea of curriculum learning refers to the idea of starting with easier examples and gradually moving to more challenging examples. As a precursor to the idea of curriculum learning in neural networks, \citet{elman1993, rohde1999} explore learning with withheld examples and/or increasing network size with divergent results. \citet{sanger1994} present one of the first examples of curriculum learning in robotic control problems. \citet{bengio2009} formalize the idea of introducing examples of increasing difficulty and call it curriculum learning. \citet{kumar2010self} propose self-paced learning, an algorithm designed to mitigate the problem of learning algorithms getting stuck in poor local optima. \citet{narvekar2020curriculumlearningreinforcementlearning}, \citet{wang2021surveycurriculumlearning} and \citet{soviany2022curriculumlearningsurvey} present surveys on curriculum learning. Results show that the order of training examples matters and that generally, incremental learning algorithms can benefit when training examples are ordered in increasing difficulty. The main conclusion from these and subsequent works in curriculum learning is that starting small and simple and gradually increasing the difficulty of the task can lead to faster convergence as well as increased performance on a task. 

\citet{parashar2025curriculumreinforcementlearningeasy} introduce the E2H Reasoner, a curriculum reinforcement learning (CRL) framework designed to improve the reasoning abilities of small-scale language models. The method works by decomposing complex problems into a sequence of tasks with increasing difficulty and using probabilistic schedulers to manage the training progression from easy to hard. Empirical results show that E2H Reasoner significantly enhances the performance of small LLMs (1.5B to 3B parameters) on reasoning and planning tasks, challenging the belief that these models are incapable of complex reasoning. The authors provide theoretical analysis supporting that this curriculum approach can be more sample-efficient than direct reinforcement learning on difficult tasks. 

\subsection{Reinforcement Learning with AI Feedback (RLAIF)}
The most straightforward solution to RLHF's main problem, its reliance on human data labelers, is to replace the human with an AI. This approach, called Reinforcement Learning with AI Feedback (RLAIF), was first tested by \citet{bai2022constitutional}. The core idea of RLAIF is to use a powerful AI model, instead of a person, to generate the preference data (e.g., "response A is better than B") needed to train a reward model or policy. The main goal is to solve the main problems with RLHF: the high resource use, slow pace, and difficulty of scaling up human data collection. 

In \citet{bai2022constitutional}, the critic model (or feedback model) is provided with a \enquote{constitution} which is a set of explicit, human-written principles to guide its feedback generation. This shifts the locus of human oversight from labeling individual data points to defining the general rules of desired behavior. The CAI process unfolds in two primary stages:
\begin{enumerate}
    \item Supervised Learning (SL) Phase: This stage aims to bootstrap the model's alignment. The model is prompted to generate responses, particularly for potentially harmful queries. It is then asked to critique its own response based on a principle from the constitution (e.g., "Identify how this response could be harmful") and subsequently revise it. This process of self-critique and revision generates a dataset of improved responses. This dataset is then used to supervised fine-tune the model.
    \item Reinforcement Learning (RL) Phase: This is the RLAIF stage proper. The SL-tuned model generates pairs of responses. A separate AI evaluator (critic model), guided by the constitution, provides preference labels for these pairs (e.g., "Choose the response that is less harmful"). This AI-generated preference dataset is used to train a preference model. Finally, the policy is optimized using an RL algorithm, with the constitution-aligned preference model providing the signal.
\end{enumerate}
\begin{table}[h!]
\label{tab:rlhf_vs_rlaif}
\centering
\begin{tabularx}{\textwidth}{|l|X|X|} 
\hline 
\multicolumn{1}{|c|}{\textbf{Feature}} & \multicolumn{1}{c|}{\textbf{RLHF}} & \multicolumn{1}{c|}{\textbf{RLAIF}} \\ 
\hline
Feedback Source & Human Annotators & AI Model\\
\hline
Scalability & Low; limited by human availability and speed & High; fully automated and continuous \\
\hline
Resources & High; human labor is resource intensive & Low; API calls are more efficient \\
\hline
Speed & Slow; annotation cycles can take weeks & Fast; rapid iteration is possible \\
\hline
Feedback Nuance & High; captures subjective, cultural, and contextual subtleties & Algorithmic; limited by the labeler's capabilities and biases \\
\hline
Primary Bias Profile & High-noise, low-bias: Inconsistent but diverse human judgments & Low-noise, high-bias: Consistent but reflects the AI's systemic biases \\
\hline
Key Challenge & Data collection logistics & Aligning the AI labeler and preventing bias amplification \\
\hline
\end{tabularx}
\caption{A comparative overview of the fundamental trade-offs between the RLHF and RLAIF approaches.}
\end{table}

In another work, \citet{lee2024rlaifvsrlhfscaling} demonstrate that RLAIF could achieve performance comparable to RLHF. For tasks like summarization and helpful dialogue, human evaluators showed no statistically significant preference between outputs from RLAIF-trained and RLHF-trained models. In the critical domain of harmlessness, the same study found that RLAIF actually outperformed RLHF, achieving higher harmlessness ratings from human evaluators.

\citet{li2025curriculum} apply curriculum learning to improve the reward model by training it on preference pairs that increase in difficulty, which they found helps mitigate issues of distribution shift. Hybrid RLAIF \cite{li2024hrlaif} was developed in response to findings that while basic RLAIF increased harmlessness, it sometimes led to a decrease in correctness; the study suggests using a hybrid of AI and human feedback to maintain helpfulness. Meanwhile, Multi-objective RLAIF \cite{williams2024morlaif} trains separate, smaller preference models for distinct principles like toxicity and factuality, which can be combined to form a more controllable reward.

However, the RLAIF framework faces significant critiques. A critical evaluation by \citet{sharma2024critical} questioned the necessity of the complex RL step altogether. Their work suggests that many of the observed gains from RLAIF may be an artifact of using a much stronger "critic" model to provide feedback. They demonstrated that simply performing supervised fine-tuning with the stronger model as the teacher could outperform the full RLAIF pipeline, implying that the benefits might stem from better data, not the RL process itself.

Bias amplification is a key concern for RLAIF. While human feedback is often described as \enquote{high-noise, low-bias}, AI feedback is the opposite: \enquote{low-noise, high-bias}. The AI labeler provides consistent feedback, but if its underlying principles are flawed, RLAIF will consistently and efficiently instill those flaws into the main model \cite{bai2022constitutional}. Moreover, the efficacy of CAI is entirely contingent on the quality of its constitution. This has spawned a new, third-order research direction focused on the constitutional design problem. How should principles be selected and phrased? How do different principles interact? Recent work has begun to tackle this \enquote{meta-alignment} challenge. \citet{kyrychenko2025c3ai} explore frameworks like C3AI to systematically craft and evaluate constitutions before fine-tuning. \citet{glaese2022inverse} explore Inverse Constitutional AI, which attempts to automatically derive a human-readable constitution from an existing dataset of human preferences, thereby making the underlying values of a dataset more transparent.

\subsection{Reinforcement learning for reasoning}
While RLAIF addresses the scalability of supervision, another frontier of research is tackling its granularity. For tasks that require complex, multi-step reasoning, such as solving mathematical proofs or writing code, rewarding only the final outcome is an inefficient signal. This has led to a critical shift from outcome supervision to process supervision, an approach that focuses on aligning the model's intermediate steps of reasoning, or its "chain of thought".

In complex reasoning tasks, a correct answer can be reached via flawed logic (lucky guess), while a correct reasoning process can be derailed by a minor error. \citet{uesato2022solvingmathwordproblems} coin the terms Outcome supervised reward models (ORMs), which provide a single reward based on the final result, and Process supervised reward models (PRMs), which provide feedback for each intermediate step in a model's generated chain of thought. They find that while supervising only the final answer (outcome-based) is sufficient for achieving a low final-answer error rate, it is not enough to guarantee correct reasoning. To ensure the step-by-step reasoning is also correct (low trace error), process-based supervision is necessary. The researchers found that reward models trained only on final outcomes can surprisingly learn to approximate process-based feedback, which helps reduce reasoning errors.

\citet{lightman2023process} conduct another study comparing ORMs and PRMs with a more powerful base model, more reasoning annotation and a more challenging MATH test data. Key findings from the paper include:
\begin{itemize}
    \item Improved Performance: A model trained with process supervision, which provides feedback on each step of a problem-solving process, solved 78.2\% of problems from a challenging math dataset. This outperformed models trained with outcome supervision, which only gives feedback on the final answer.
    \item Greater Reliability: Process-supervised reward models (PRMs) were better at identifying correct solutions, especially as the number of possible solutions increased. This indicates that PRMs are more effective for searching through many potential answers.
    \item Active Learning Efficiency: The study demonstrated that "active learning," a strategy for selecting the most informative problems for human review, can make the process of training with human feedback 2.6 times more efficient.
    \item Handling Flawed Reasoning: Process supervision helps models avoid reaching the correct answer through incorrect reasoning, a common issue with outcome-supervised models.
\end{itemize}

Building on this, \citet{li2024pspo} introduce PSPO (Process-Supervised Policy Optimization), arguing that the reward for a reasoning chain is a nonlinear function of both its accuracy and its length. Reasoning that is too terse may be incomplete, while reasoning that is too verbose may be redundant and introduce noise. PSPO proposes a universal framework for nonlinear reward shaping. \citet{guan2025deliberativealignmentreasoningenables} extends process supervision directly into the domain of safety, with deliberative alignment. This approach explicitly trains a model to reason through the text of human-written safety specifications before generating a response. The model's chain of thought, which includes this deliberation over safety rules, is directly supervised during SFT and then refined with RL. This approach offers a highly scalable method for safety alignment, as it can be used to generate vast amounts of synthetic training data. This moves safety alignment from simply avoiding bad outputs to actively reasoning about and applying safety principles, a much more robust and generalizable form of safety. \citet{gao2025flower} propose using generative flow networks (GFlowNets) to provide process-level supervision. This approach assigns token-level rewards to align token generation probabilities with their reward signal, helping mitigate popularity bias and enhancing fairness and diversity. \citet{mou2025saroenhancingllmsafety} introduce Safety-oriented Reasoning Optimization framework which applies a two-stage process-based approach to safety, first using supervised fine-tuning to warm up the model's ability to produce long-chain reasoning, and then using a process-based optimization stage to align that reasoning with specific safety policies.

If this approach of aligning the reasoning process generalizes beyond mathematics and safety, it suggests a future where the distinction between enhancing capabilities and ensuring alignment begins to dissolve. Training a model to become a better reasoner and training it to be safer and more interpretable could become two sides of the same coin, with the potential to minimize the safety tax and create a unified training process.

\subsection{Game theory and self-play}
Game theory and reinforcement learning have a long and intertwined history. Among other things, game theory provides a natural mathematical framework for multi-agent reinforcement learning problems. \citet{wu2024selfplaypreferenceoptimizationlanguage} argue that standard methods for reinforcement learning from human feedback (RLHF) don't fully capture the complexities of human preferences, which can be inconsistent or irrational. To address this, SPPO treats language model alignment as a two-player game, aiming to find the best possible policy, known as the Nash equilibrium. The method uses an iterative process where a language model plays against a previous version of itself to improve. A key contribution is a new and theoretically-grounded objective for this process that is also simple and effective. Its objective function is more expressive than DPO's; while DPO's loss only maximizes the log-probability gap between chosen and rejected responses, SPPO's objective can simultaneously increase the log-likelihood of preferred responses while actively decreasing the log-likelihood of dispreferred ones.

In experiments, the researchers fine-tuned the Mistral-7B-Instruct-v0.2 and Llama-3-8B-Instruct models using the SPPO method. They used prompts from the UltraFeedback dataset and a pre-trained preference model called PairRM. The results showed that SPPO significantly outperformed other methods like Direct Preference Optimization (DPO) and Identity Preference Optimization (IPO) on several benchmarks, including Alpaca Eval 2.0, MT-Bench, and Arena-Hard. Notably, a fine-tuned Mistral-7B model achieved a state-of-the-art length-controlled win-rate of 28.53\% against GPT-4-Turbo on Alpaca Eval 2.0, and a fine-tuned Llama-3-8B model reached a win rate of 38.77\%. This strong performance was achieved without needing direct supervision from more advanced models like GPT-4.

The broader concept of self-alignment through interaction is being explored in other contexts as well, such as simulating multi-agent social scenarios to teach a model about the downstream consequences of its responses, or using self-generated role-play dialogues to improve a model's character-impersonation abilities. \citet{pang2024selfalignmentlargelanguagemodels} introduce MATRIX, a self-alignment framework where a large language model (LLM) simulates social scenes to understand the consequences of its potential responses. Through a "monopolylogue" where the LLM plays all relevant roles, it generates consequence-aware data that is then used to fine-tune the model itself, achieving value alignment without external supervision. The authors demonstrate that their tuned 13B model surpasses GPT-4 in alignment with human values, as rated by human evaluators. \citet{lu2024largelanguagemodelssuperpositions} introduce DITTO, a self-alignment method that enhances the role-playing capabilities of large language models by avoiding distillation from proprietary models. The proposed method leverages knowledge augmentation from public databases to simulate role-play dialogues, reformulating the task as a reading comprehension exercise to fine-tune the model. This approach was shown to significantly improve the consistency of role identity and the use of accurate and role-specific knowledge, achieving performance comparable to advanced proprietary chatbots.

As LLMs achieve superhuman capabilities, it will become impossible for humans to provide reliable supervision. Debate is a mechanism proposed for scalable oversight, designed to allow weaker supervisors to evaluate the outputs of stronger expert AIs. The premise is that it is easier to identify the stronger argument in a competitive debate than it is to generate the correct answer. In a typical debate, two expert models argue for opposing answers to a question, and a judge declares a winner based on the transcript. A study by \citet{khan2024debating} found that this adversarial setup significantly improves the accuracy of both human and AI judges. A key finding from their research is that optimizing debaters to be more persuasive also leads to more truthful outcomes, suggesting that truth may be inherently easier to argue for. This provides a potential unsupervised signal for alignment: training models to win debates may also train them to be more truthful. \citet{arnesen2024traininglanguagemodelswin} show that that training models to win debates via self-play can also improve judge accuracy.

These methods blur the line between alignment and capabilities research, suggesting a future where the most powerful techniques for controlling AI systems may be the very same adversarial and self-referential processes used to develop their intelligence.

\subsection{Other improvements in offline policy optimization}
Although DPO represents a significant simplification over on-policy RL algorithms, its reliance on a specific data format, pairwise preferences, has limitations for certain types of tasks. This has spurred a re-convergence with the broader field of deep reinforcement learning, as researchers adapt more traditional RL concepts to create a new generation of offline algorithms tailored for the unique challenges of LLM alignment, particularly in complex reasoning domains.

\citet{wang2024offlinereinforcementlearningllm} introduce Offline Reasoning Optimization (OREO), an offline reinforcement learning algorithm designed to enhance the multi-step reasoning of large language models. The method addresses the limitations of Direct Preference Optimization (DPO), which is less suitable for reasoning tasks due to its reliance on paired preference data and ineffective credit assignment with sparse rewards. OREO overcomes this by jointly learning a policy and a value function through a technique which does not require pairwise data. The authors demonstrate that OREO outperforms existing offline learning methods on mathematical reasoning (GSM8K, MATH) and embodied agent control (ALFWorld) tasks. Furthermore, the value function learned during training can be used to guide the tree search during inference for improved performance at no additional training effort.

The development of OREO is part of a broader trend of creating more sophisticated offline RL algorithms for LLMs. \citet{brantley2025acceleratingrlllmreasoning} propose A*-PO, a two-stage policy optimization framework designed to accelerate the reinforcement learning (RL) process for fine-tuning large language models on reasoning tasks. The method first estimates the optimal value function offline, then uses a simple regression loss for on-policy updates, requiring only a single generation per prompt. This approach achieves comparable or superior performance to existing methods like PPO and GRPO while reducing training time by up to 2x and peak memory usage by over 30\%.

\citet{roux2025taperedoffpolicyreinforcestable} propose Tapered Off-Policy REINFORCE (TOPR), an algorithm that uses an asymmetric, tapered form of importance sampling to stably and efficiently fine-tune large language models with reinforcement learning. The method effectively leverages both positive and negative examples to improve performance without requiring KL regularization.

\section{New research ideas}\label{sec:newresearchideas}
In this section, we present some new ideas to investigate. These ideas are presented as a single new methodology; however, it is simply a set of techniques that can be implemented and investigated independently.

\subsection{Generalized relative advantage policy evolution (GRAPE)}
One of the challenges of preference rankings is that the reward model must learn to pick the preferred response on a variety of prompts in the data with different expectations. Here we propose an approach that combines the elements of RLHF and RLAIF, eliminating the need for training a value model or reward model, but still optionally allowing for fine-tuned reward models to incorporate human feedback.

Let's consider the sample of questions; let's say we have $Q$ questions as discussed before. Let's group each of the $Q$ questions into capability or question type e.g. coding, math, safety, etc.. such that we have $K$ categories of questions. Next, we do the following things:

\begin{enumerate}
    \item Write a system prompt for the model to generate the best responses for each category of questions, e.g. \enquote{You are a highly skilled and professional coding assistant. Your primary function is to help users with a wide range of programming tasks...} \label{enit:simpleresponsegenerate}
    \item Write a system prompt for the model to generate an improved response to a given question provided the response, and a critique with reasoning for the response score.\label{enit:responseiterate}
    \item Write a separate rubric for each category to evaluate the response. For example, for coding questions:
\begin{itemize}
    \item \emph{Best Practices}: The code follows common conventions and best practices for the specified programming language (e.g., Python's PEP 8).
    \item \emph{Valid libraries}: The response does not invent non-existent libraries, functions, or methods.
    \item \emph{Correctness}: The provided code runs without errors and produces the correct result for standard test cases.
    \item \emph{Completeness}: The solution addresses all parts of the user's question and respects all given constraints.
    \item \emph{Handles Edge Cases}: The solution correctly processes edge cases, such as empty inputs, null values, or zeros.
    \item ...
\end{itemize}
\item A weight for each rubric item, providing a sense of relative importance of that item within the rubric of that category. This is needed because some rubric items may be foundational to a response e.g. checking correctness of a math question is likely more important than clear styling of the proof. As such any aggregation of individual rubric scores to get a final score for the response will need to be a weighted one. Suppose a category has $\rho$ total number of individual rubric items in the rubric, then let the weights be $\omega_1,\omega_2,\ldots,\omega_{\rho}$.
\item For each rubric item of each category, write an \enquote{scoring flow} based on whether the item is verifiable or non-verifiable:
\begin{enumerate}
    \item \emph{Verifiable rubric items}: Where there is a verifiable ground truth (e.g. correct answer to a math problem), ask a model to compare the response to the ground truth and return a correctness score between 0 and 1 in this format:\begin{enumerate}\label{enumitem:confscore}
        \item Reasoning for the score
        \item Score
        \item Confidence in the score
    \end{enumerate}
    \item \emph{Non-verifiable rubric items}: Non-verifiable conditions can sometimes be decomposed into verifiable conditions with some residual non-verifiable parts. We call this \enquote{atomization} of rubric. An example of this is when we are trying to check if a code appropriately handles edge cases. One can add a number of verifiable rubrics such as \enquote{Ensure the program does not crash when input is null} to reduce the non-verifiable surface area of the original rubric item. We recommend atomization where possible. Nonetheless, you are always left with non-verifiable rubric items. In these cases, the scoring flow looks something like this:
    \begin{enumerate}
        \item Prepare a critique system prompt for the rubric item, something like \enquote{You are a highly experienced and meticulous code style critic. Your role is to analyze a given code block and provide a detailed, objective critique of its stylistic qualities...}. We provide an example of this in Appendix~\ref{sec:critiquesysprompt}
        \item Provide the question, the response, the rubric, and the critique to a model to return a success score between 0 and 1 in the same format as in \ref{enumitem:confscore} above.
    \end{enumerate}
\end{enumerate}
\end{enumerate}

That's a lot of upfront work. But all of this allows us to create pipeline that takes us from generation to scoring of the responses. Let's write down how we will use these. As noted before we have $Q$ questions and let's say we generate $G$ responses for each question, thus giving us $S=QG$ samples if $\vec{text}$ where each $\vec{text}$ is the concatenation of the question and the response.

First thing to note is that the $G$ responses for each question can be generated in two ways:
\begin{itemize}
    \item By using the model and system prompt in \ref{enit:simpleresponsegenerate} above
    \item By providing a response from previous iteration and using \ref{enit:responseiterate} above
\end{itemize}
While for the first set of responses we would need to use \ref{enit:simpleresponsegenerate}, for further generations we recommend choosing between \ref{enit:simpleresponsegenerate} and \ref{enit:responseiterate} with some probability, which can be tweaked through experimentation. You could also have a fixed proportion of the $G$ generations coming from each of the two methods.

Once we have the generations, we can obtain the i) reasoning, ii) score, and iii) confidence for all the rubric items for all samples. Now we need to aggregate this information in a way so as to write a differentiable loss function that can help us update model parameters. Let's take stock of what we have so far:
\begin{itemize}
    \item We have $S=QG$ samples of question-response pairs, i.e., $\vec{text}$, where each one belongs to one of the categories $K$
    \item We have a rubric applicable to each sample $\vec{text}$ (depending on the category the question belongs to). Let's say for $\vec{text}_i$ there are $\rho_i$ individual rubric items, then we also have weights for each item $\omega_1,\omega_2,\ldots,\omega_{\rho_i}$ that add up to one.
    \item For each $\vec{text}_i$ we have the reasoning string $\Psi_{ij}$, the score $\tau_{ij}$ and the score confidence $\varphi_{ij}$ for the $j^{th}$ rubric item corresponding to the category of that question
\end{itemize}

Now, Let $\xi(\vec{text}_i)$ represent the set of the indices of all samples in $S$ such that they belong to the same category as $\vec{text}_i$ does. This means that for any $k\in \xi(\vec{text}_i)$ the questions from which $\vec{text}_i$ and $\vec{text}_k$ were generated belong in the same category in $K$. They may or may not be the same question, but the questions are of the same type, e.g. if $\vec{text}_i$ comes from a math question then $\xi(\vec{text}_i)$ is the set of all math samples. Let $\bar{\varphi}_{ij}$ be the average confidence of a particular rubric item across all samples:
\begin{align*}
    \bar{\varphi}_{ij} &=    \frac{\sum_{k \in \xi(\vec{text}_i)} \varphi_{kj}}{|\xi(\vec{text}_i)|}
\end{align*}
where $|\xi(\vec{text}_i)|$ denotes the number of elements in $\xi(\vec{text}_i)$. This means $\bar{\varphi}_{ij} = \bar{\varphi}_{kj}$ for all $i$ and $k$ in the same category (i.e. they have the same evaluation rubric).

Instead of defining the loss from start, let's start by defining an advantage function that can be used with the any of the existing update algorithms (e.g. PPO). First, we will aggregate the scores into a reward function of sorts.
\begin{align}
    R(\vec{text}_i) &=  \frac{\rho_i\sum_{j=1}^{\rho_i}\omega_{ij} \tau_{ij} \bar{\varphi_{ij}}}{\sum_{j=1}^{\rho_i} \bar{\varphi_{ij}}} \\
\end{align}
Using this reward function, we can write the advantage function same as in GRPO:
\begin{align*}
    A(\vec{text}_{qg}) = R(\vec{text}_{qg}) - \frac{1}{G}\sum_{i=1}^G R(\vec{text}_{qi})
\end{align*}
And the loss function can similarly be the same as GRPO/PPO when written in terms of the advantage function. This advantage function is not defined at the token level (same as GRPO), and only at the level of the whole sample. One potential experiment to conduct would be to take advantage of this and simplify the PPO loss to apply clipping at the aggregate sample level rather than token level, something like this:
\begin{align}\label{eq:altgrapeloss}
\text{Alt. GRAPE Loss} = - \frac{1}{S}\sum_{i=1}^{S} \min & \left[ \frac{\pi_1(\vec{text}_{i})}{\pi_0(\vec{text}_{i})}  A(\vec{text}_{i}) ,\right.\\
&  \left. \ \ \ \text{CLIP}\left( \frac{\pi_1(\vec{text}_{i})}{\pi_0(\vec{text}_{i})}, 1-\epsilon,1+\epsilon \right)A(\vec{text}_{i}) \right]\nonumber
\end{align}
Though if you are experimenting with GRAPE, I would start with a more direct implementation with standard PPO package (same as GRPO) rather than attempt to make all the changes at the same time, especially since this loss function can increase the chances of model collapse due to the relaxed clipping constraint.

GRAPE is a generalized framework for reinforcement learning for model alignment. The GRPO algorithm can be considered a special case of GRAPE with a single rubric (human preference) with the reward model being used as the scoring model with no reasoning and assuming all confidence scores as 1. In general, GRAPE allows for numerous adaptations and iterative improvements. Here is a discussion of the features of GRAPE:
\begin{itemize}
    \item \emph{Human Feedback Integration:} Human preferences can be incorporated by creating a specific rubric item for it. A standard reward model, trained on human rankings, can then serve as the scoring model for that item, providing the reasoning, score, and confidence level.
    \item \emph{Use of SFT Data:} Existing Supervised Fine-Tuning (SFT) data can be seamlessly included by treating SFT responses as one of the multiple responses generated for a given question during the evaluation process.
    \item \emph{Continuous Improvement:} The framework is built for iteration. Key components like evaluation rubrics and the system prompts used for generation can be continuously updated and refined to enhance model performance over time. New rubric items can be easily added.
    \item \emph{Iterate on response generation:} One can experiment with and improve generation strategies without altering the scoring mechanism, and vice-versa.
    \item \emph{Reward fine tuning:} The framework allows for fine-tuning the scoring models. This means over time capabilities of interest can be further researched and improved by fine-tuning corresponding models with human data.
    \item \emph{Reuse of critique and reasoning:} GRAPE efficiently uses critique and reasoning for scoring by allowing them to be reused for follow-up generations in a loop.
\end{itemize}
In summary, GRAPE's structure breaks down model alignment into manageable parts, enabling targeted, independent, and continuous improvement of language models.

\subsubsection{Why use confidence when aggregating scores?}
When we ask the scoring model to provide a score we ask for three things:
\begin{enumerate}
    \item Reasoning for the score
    \item Score
    \item Confidence in the score
\end{enumerate}
One of the things we do is use the confidence scores when aggregating the scores for each rubric item into a single score for the sample. Why do we do this? When we ask models to provide confidence in rubric scores, low confidence is an indicator of a poor quality rubric item. Generally, one can identify poor rubric items for manual review by looking at corresponding average confidence scores for those items. As such, by assuming that confidence is inversely proportional to the score variance, we can improve the variance of the estimate of the underlying capability. We show a weaker result in Lemma~\ref{lemma:confweighted}. This is because variance does not reduce for all possible values of rubric weights.

\section{Conclusion}
The journey to align large language models with human intent has evolved rapidly, moving from the straightforward objective of supervised fine-tuning to the sophisticated dynamics of reinforcement learning. This paper has traced this evolution, starting with foundational techniques like SFT and rejection sampling, and progressing to the development of robust reinforcement learning from human feedback (RLHF) frameworks. Methods like PPO and DPO have become cornerstones of modern alignment, providing the tools to steer models using nuanced human preferences while mitigating critical issues like model collapse.

\begin{table}[h!]
\label{tab:rl_approaches}
\centering
\begin{tabularx}{\textwidth}{|l|X|X|X|} 
\hline 
\multicolumn{1}{|c|}{\textbf{Approach}} & \multicolumn{1}{c|}{\textbf{Core Objective}} & \multicolumn{1}{c|}{\textbf{Algorithms}}  & \multicolumn{1}{c|}{\textbf{Solves for}} \\ 
\hline
RLAIF / CAI & Scale supervision by replacing human feedback with principled AI feedback. & Constitutional AI, d-RLAIF, Curriculum-RLAIF, HRLAIF & Effort, speed, and scalability of RLHF. \\
\hline
Process Supervision & Align the intermediate reasoning steps, not just the final output. & PRMs, PSPO*, Deliberative Alignment & Poor credit assignment and uninterpretability of outcome-based rewards in complex reasoning. \\
\hline
Multi-Agent / Self-Play & Generate alignment signals through dynamic, interactive processes. & Debate, Self-Play Preference Optimization (SPPO) & Need for ground-truth labels in superhuman domains; intransitivity of simple preference models. \\
\hline
Advanced Offline RL & Learn from offline datasets with sparse, terminal rewards. & OREO, A*-PO, TOPR & Inapplicability of preference-pair methods like DPO to many tasks. \\
\hline
\end{tabularx}
\caption{A summary of emerging alignment approaches, their core objectives, and the limitations they are designed to address.}
\end{table}

However, as models grow in capability, the limitations of these established methods have become apparent, pushing the research frontier. The future of model alignment is being shaped by a move away from static, human-intensive supervision and toward more scalable, granular, and dynamic approaches. These emerging approaches represent a fundamental shift in how we think about alignment. The research frontier is moving along three interconnected axes: from human to AI-driven supervision, from holistic outcomes to fine-grained processes, and from static datasets to dynamic, interactive environments. RLAIF addresses the critical bottleneck of scalability by leveraging AI to generate feedback. Process supervision tackles the challenge of complex reasoning by rewarding the 'how' not just the 'what,' ensuring that models arrive at correct answers through sound logic. At the same time, game theory and self-play offer a path to supervising models that exceed human capabilities, using adversarial setups like debate to generate reliable training signals. Finally, a new generation of advanced offline RL algorithms is being developed to be more efficient and better suited to the unique constraints of LLM training.

In this context, we introduced GRAPE (Generalized Relative Advantage Policy Evolution), a novel framework designed to synthesize these advancements. By using detailed, category-specific rubrics and integrating AI-generated critiques, GRAPE offers a modular and transparent pathway for continuous model improvement, eliminating the need for separate value and reward models while retaining the flexibility to incorporate human oversight.

Ultimately, one of the things that is becoming more clear is that the problem of model alignment is possibly best addressed by making models better at reasoning, rather than treating it as a sort of a balancing act between various capabilities (e.g. math vs safety).

\bibliographystyle{chicago}
\bibliography{references}

\begin{thebibliography}{}

\bibitem[\protect\citeauthoryear{Arnesen, Rein, and Michael}{Arnesen et~al.}{2024}]{arnesen2024traininglanguagemodelswin}
Arnesen, S., D.~Rein, and J.~Michael (2024).
\newblock Training language models to win debates with self-play improves judge accuracy.

\bibitem[\protect\citeauthoryear{Askell, Bibarts, guest, Hestness, Lopyrev, downstairs, upstairs, character, Clark, Drain, et~al.}{Askell et~al.}{2021}]{askell2021general}
Askell, A., A.~Bibarts, a.~guest, J.~Hestness, K.~Lopyrev, j.~downstairs, j.~upstairs, j.~character, T.~Clark, D.~Drain, et~al. (2021).
\newblock A general language assistant as a laboratory for alignment.
\newblock {\em arXiv preprint arXiv:2112.00861\/}.

\bibitem[\protect\citeauthoryear{Bai, Kadavath, Kundu, Askell, Kernion, Jones, Chen, Goldie, Mirhoseini, McKinnon, Chen, Olsson, Olah, Hernandez, Drain, Ganguli, Li, Tran-Johnson, Perez, Kerr, Mueller, Ladish, Landau, Ndousse, Lukosuite, Lovitt, Sellitto, Elhage, Schiefer, Mercado, DasSarma, Lasenby, Larson, Ringer, Johnston, Kravec, Showk, Fort, Lanham, Telleen-Lawton, Conerly, Henighan, Hume, Bowman, Hatfield-Dodds, Mann, Amodei, Joseph, McCandlish, Brown, and Kaplan}{Bai et~al.}{2022}]{bai2022constitutional}
Bai, Y., S.~Kadavath, S.~Kundu, A.~Askell, J.~Kernion, A.~Jones, A.~Chen, A.~Goldie, A.~Mirhoseini, C.~McKinnon, C.~Chen, C.~Olsson, C.~Olah, D.~Hernandez, D.~Drain, D.~Ganguli, D.~Li, E.~Tran-Johnson, E.~Perez, J.~Kerr, J.~Mueller, J.~Ladish, J.~Landau, K.~Ndousse, K.~Lukosuite, L.~Lovitt, M.~Sellitto, N.~Elhage, N.~Schiefer, N.~Mercado, N.~DasSarma, R.~Lasenby, R.~Larson, S.~Ringer, S.~Johnston, S.~Kravec, S.~E. Showk, S.~Fort, T.~Lanham, T.~Telleen-Lawton, T.~Conerly, T.~Henighan, T.~Hume, S.~R. Bowman, Z.~Hatfield-Dodds, B.~Mann, D.~Amodei, N.~Joseph, S.~McCandlish, T.~Brown, and J.~Kaplan (2022).
\newblock Constitutional ai: Harmlessness from ai feedback.

\bibitem[\protect\citeauthoryear{Baird}{Baird}{1993}]{Baird1993}
Baird, L.~C. (1993).
\newblock Advantage updating.
\newblock Technical Report WL-TR-93-1146, Wright Laboratory, Wright-Patterson Air Force Base.

\bibitem[\protect\citeauthoryear{Barto, Sutton, and Anderson}{Barto et~al.}{1983}]{barto1983neuronlike}
Barto, A.~G., R.~S. Sutton, and C.~W. Anderson (1983).
\newblock Neuronlike adaptive elements that can solve difficult learning control problems.
\newblock {\em IEEE Transactions on Systems, Man, and Cybernetics\/}~{\em SMC-13\/}(5), 834--846.

\bibitem[\protect\citeauthoryear{Bellman}{Bellman}{1957}]{bellman1957dynamic}
Bellman, R. (1957).
\newblock {\em Dynamic Programming}.
\newblock Princeton University Press.

\bibitem[\protect\citeauthoryear{Bengio, Ducharme, Vincent, and Jauvin}{Bengio et~al.}{2003}]{bengio2003neural}
Bengio, Y., R.~Ducharme, P.~Vincent, and C.~Jauvin (2003).
\newblock A neural probabilistic language model.
\newblock {\em Journal of machine learning research\/}~{\em 3\/}(Feb), 1137--1155.

\bibitem[\protect\citeauthoryear{Bengio, Louradour, Collobert, and Weston}{Bengio et~al.}{2009}]{bengio2009}
Bengio, Y., J.~Louradour, R.~Collobert, and J.~Weston (2009).
\newblock Curriculum learning.
\newblock In {\em Proceedings of the 26th annual international conference on machine learning}, pp.\  41--48.

\bibitem[\protect\citeauthoryear{Boltzmann}{Boltzmann}{1877}]{boltzmann1877}
Boltzmann, L. (1877).
\newblock {\"U}ber die beziehung zwischen dem zweiten hauptsatze der mechanischen w{\"a}rmetheorie und der wahrscheinlichkeitsrechnung respective den s{\"a}tzen {\"u}ber das w{\"a}rmegleichgewicht.
\newblock {\em Sitzungsberichte der Kaiserlichen Akademie der Wissenschaften in Wien. Mathematisch-Naturwissenschaftliche Classe\/}~{\em 76}, 373--435.

\bibitem[\protect\citeauthoryear{B{\"o}rgers and Sarin}{B{\"o}rgers and Sarin}{1997}]{borgers1997learning}
B{\"o}rgers, T. and R.~Sarin (1997).
\newblock Learning through reinforcement and replicator dynamics.
\newblock {\em Journal of Economic Theory\/}~{\em 77\/}(1), 1--14.

\bibitem[\protect\citeauthoryear{Brantley, Chen, Gao, Lee, Sun, Zhan, and Zhang}{Brantley et~al.}{2025}]{brantley2025acceleratingrlllmreasoning}
Brantley, K., M.~Chen, Z.~Gao, J.~D. Lee, W.~Sun, W.~Zhan, and X.~Zhang (2025).
\newblock Accelerating rl for llm reasoning with optimal advantage regression.

\bibitem[\protect\citeauthoryear{Brown}{Brown}{1951}]{brown1951iterative}
Brown, G.~W. (1951).
\newblock Iterative solution of games by fictitious play.
\newblock In {\em Activity Analysis of Production and Allocation}, pp.\  374--376. Wiley.

\bibitem[\protect\citeauthoryear{Christiano, Leike, Brown, Martic, Legg, and Amodei}{Christiano et~al.}{2017}]{christiano2017deep}
Christiano, P.~F., J.~Leike, T.~Brown, M.~Martic, S.~Legg, and D.~Amodei (2017).
\newblock Deep reinforcement learning from human preferences.
\newblock In {\em Advances in neural information processing systems}, pp.\  4299--4307.

\bibitem[\protect\citeauthoryear{Clausius}{Clausius}{1865}]{clausius1865}
Clausius, R. (1865).
\newblock Ueber verschiedene f{\"u}r die anwendung bequeme formen der hauptgleichungen der mechanischen w{\"a}rmetheorie.
\newblock {\em Annalen der Physik\/}~{\em 201\/}(7), 353--400.

\bibitem[\protect\citeauthoryear{DeepSeek-AI}{DeepSeek-AI}{2025}]{deepseekai2025deepseekv3technicalreport}
DeepSeek-AI (2025).
\newblock Deepseek-v3 technical report.

\bibitem[\protect\citeauthoryear{Devlin, Chang, Lee, and Toutanova}{Devlin et~al.}{2019}]{devlin2019bertpretrainingdeepbidirectional}
Devlin, J., M.-W. Chang, K.~Lee, and K.~Toutanova (2019).
\newblock Bert: Pre-training of deep bidirectional transformers for language understanding.

\bibitem[\protect\citeauthoryear{Elman}{Elman}{1993}]{elman1993}
Elman, J.~L. (1993).
\newblock Learning and development in neural networks: The importance of starting small.
\newblock {\em Cognition\/}~{\em 48\/}(1), 71--99.

\bibitem[\protect\citeauthoryear{Fisher}{Fisher}{1922}]{fisher1922mathematical}
Fisher, R.~A. (1922).
\newblock On the mathematical foundations of theoretical statistics.
\newblock {\em Philosophical Transactions of the Royal Society of London. Series A, Containing Papers of a Mathematical or Physical Character\/}~{\em 222\/}(594-604), 309--368.

\bibitem[\protect\citeauthoryear{F{\"u}rnkranz, H{\"u}llermeier, Hoche, and Brinker}{F{\"u}rnkranz et~al.}{2012}]{furnkranz2012preference}
F{\"u}rnkranz, J., E.~H{\"u}llermeier, S.~Hoche, and K.~Brinker (2012).
\newblock Preference-based reinforcement learning: a formal framework and a policy iteration algorithm.
\newblock {\em Machine learning\/}~{\em 89\/}(1), 123--156.

\bibitem[\protect\citeauthoryear{Gao, Gao, Fan, Yuan, Shi, and He}{Gao et~al.}{2025}]{gao2025flower}
Gao, C., M.~Gao, C.~Fan, S.~Yuan, W.~Shi, and X.~He (2025).
\newblock Process-supervised {LLM} recommenders via flow-guided tuning.

\bibitem[\protect\citeauthoryear{Glaese, Stabin, Chen, Huang, and Irving}{Glaese et~al.}{2022}]{glaese2022inverse}
Glaese, A., F.~Stabin, Y.~Chen, P.-S. Huang, and G.~Irving (2022).
\newblock Inverse constitutional ai: Compressing preferences into principles.

\bibitem[\protect\citeauthoryear{Guan, Joglekar, Wallace, Jain, Barak, Helyar, Dias, Vallone, Ren, Wei, Chung, Toyer, Heidecke, Beutel, and Glaese}{Guan et~al.}{2025}]{guan2025deliberativealignmentreasoningenables}
Guan, M.~Y., M.~Joglekar, E.~Wallace, S.~Jain, B.~Barak, A.~Helyar, R.~Dias, A.~Vallone, H.~Ren, J.~Wei, H.~W. Chung, S.~Toyer, J.~Heidecke, A.~Beutel, and A.~Glaese (2025).
\newblock Deliberative alignment: Reasoning enables safer language models.

\bibitem[\protect\citeauthoryear{Hinton, Vinyals, and Dean}{Hinton et~al.}{2015}]{hinton2015distilling}
Hinton, G., O.~Vinyals, and J.~Dean (2015).
\newblock Distilling the knowledge in a neural network.
\newblock {\em arXiv preprint arXiv:1503.02531\/}.

\bibitem[\protect\citeauthoryear{Hopfield}{Hopfield}{1987}]{hopfield1987learning}
Hopfield, J.~J. (1987).
\newblock Learning algorithms and probability distributions in feed-forward and feed-back networks.
\newblock {\em Proceedings of the National Academy of Sciences\/}~{\em 84\/}(23), 8429--8433.

\bibitem[\protect\citeauthoryear{Hu and Wellman}{Hu and Wellman}{2003}]{hu2003nash}
Hu, J. and M.~P. Wellman (2003).
\newblock Nash q-learning for general-sum stochastic games.
\newblock {\em Journal of machine learning research\/}~{\em 4\/}(Nov), 1039--1069.

\bibitem[\protect\citeauthoryear{Kahn}{Kahn}{1949}]{kahn1949random}
Kahn, H. (1949).
\newblock Random sampling (monte carlo) techniques in neutron attenuation problems.
\newblock Technical report, RAND Corporation.

\bibitem[\protect\citeauthoryear{Kahn and Harris}{Kahn and Harris}{1951}]{kahn1951estimation}
Kahn, H. and T.~E. Harris (1951).
\newblock Estimation of particle transmission by random sampling.
\newblock In A.~S. Householder, G.~E. Forsythe, and H.~H. Germond (Eds.), {\em Monte Carlo Method}, pp.\  1--14. U.S. Government Printing Office.

\bibitem[\protect\citeauthoryear{Kakade and Langford}{Kakade and Langford}{2002}]{kakade2002approximately}
Kakade, S. and J.~Langford (2002).
\newblock Approximately optimal approximate reinforcement learning.
\newblock In {\em Proceedings of the Nineteenth International Conference on Machine Learning}, pp.\  267--274.

\bibitem[\protect\citeauthoryear{Kakade}{Kakade}{2001}]{kakade2001natural}
Kakade, S.~M. (2001).
\newblock A natural policy gradient.
\newblock In {\em Advances in neural information processing systems}, pp.\  1531--1538.

\bibitem[\protect\citeauthoryear{Khan, Hughes, Valentine, Ruis, Sachan, Radhakrishnan, Grefenstette, Bowman, Rocktäschel, and Perez}{Khan et~al.}{2024}]{khan2024debating}
Khan, A., J.~Hughes, D.~Valentine, L.~Ruis, K.~Sachan, A.~Radhakrishnan, E.~Grefenstette, S.~R. Bowman, T.~Rocktäschel, and E.~Perez (2024).
\newblock Debating with more persuasive {LLMs} leads to more truthful answers.

\bibitem[\protect\citeauthoryear{Kullback and Leibler}{Kullback and Leibler}{1951}]{kullback1951information}
Kullback, S. and R.~A. Leibler (1951).
\newblock On information and sufficiency.
\newblock {\em The Annals of Mathematical Statistics\/}~{\em 22\/}(1), 79--86.

\bibitem[\protect\citeauthoryear{Kumar, Packer, and Koller}{Kumar et~al.}{2010}]{kumar2010self}
Kumar, M., B.~Packer, and D.~Koller (2010).
\newblock Self-paced learning for latent variable models.
\newblock {\em Advances in neural information processing systems\/}~{\em 23}.

\bibitem[\protect\citeauthoryear{Kyrychenko, Zhou, Bogucka, and Quercia}{Kyrychenko et~al.}{2025}]{kyrychenko2025c3ai}
Kyrychenko, Y., K.~Zhou, E.~Bogucka, and D.~Quercia (2025).
\newblock {C3AI}: Crafting and evaluating constitutions for constitutional {AI}.

\bibitem[\protect\citeauthoryear{Lee, Phatale, Mansoor, Mesnard, Ferret, Lu, Bishop, Hall, Carbune, Rastogi, and Prakash}{Lee et~al.}{2024}]{lee2024rlaifvsrlhfscaling}
Lee, H., S.~Phatale, H.~Mansoor, T.~Mesnard, J.~Ferret, K.~Lu, C.~Bishop, E.~Hall, V.~Carbune, A.~Rastogi, and S.~Prakash (2024).
\newblock Rlaif vs. rlhf: Scaling reinforcement learning from human feedback with ai feedback.

\bibitem[\protect\citeauthoryear{Li, He, Sohn, Myaeng, Tu, Mak, Yu, Lee, Lee, Kim, Yang, Kim, Choi, Kim, and Lee}{Li et~al.}{2024}]{li2024hrlaif}
Li, A., Z.~He, H.-G. Sohn, S.-H. Myaeng, Z.~Tu, M.-W. Mak, C.-H. Yu, J.-Y. Lee, D.-Y. Lee, J.-H. Kim, M.-C. Yang, M.-K. Kim, S.-H. Choi, S.-H. Kim, and J.-W. Lee (2024).
\newblock {HRLAIF}: Improvements in helpfulness and harmlessness in open-domain reinforcement learning from {AI} feedback.
\newblock {\em arXiv preprint arXiv:2403.08309\/}.

\bibitem[\protect\citeauthoryear{Li, Liang, Zhang, Yang, Feng, and Gao}{Li et~al.}{2024}]{li2024pspo}
Li, J., X.~Liang, J.~Zhang, Y.~Yang, C.~Feng, and Y.~Gao (2024).
\newblock {PSPO\*}: An effective process-supervised policy optimization for reasoning alignment.

\bibitem[\protect\citeauthoryear{Li, Lin, Zhao, Lu, Zhao, Wermter, and Wang}{Li et~al.}{2025}]{li2025curriculum}
Li, M., J.~Lin, X.~Zhao, W.~Lu, P.~Zhao, S.~Wermter, and D.~Wang (2025).
\newblock {Curriculum-RLAIF}: Curriculum alignment with reinforcement learning from {AI} feedback.

\bibitem[\protect\citeauthoryear{Lightman, Kosaraju, Burda, Edwards, Baker, Lee, Leike, Schulman, Sutskever, and Cobbe}{Lightman et~al.}{2023}]{lightman2023process}
Lightman, H., V.~Kosaraju, Y.~Burda, H.~Edwards, B.~Baker, T.~Lee, J.~Leike, J.~Schulman, I.~Sutskever, and K.~Cobbe (2023).
\newblock Let's verify step by step.

\bibitem[\protect\citeauthoryear{Littman}{Littman}{1994}]{littman1994markov}
Littman, M.~L. (1994).
\newblock Markov games as a framework for multi-agent reinforcement learning.
\newblock In {\em Machine Learning Proceedings 1994}, pp.\  157--163. Elsevier.

\bibitem[\protect\citeauthoryear{Lu, Yu, Zhou, and Zhou}{Lu et~al.}{2024}]{lu2024largelanguagemodelssuperpositions}
Lu, K., B.~Yu, C.~Zhou, and J.~Zhou (2024).
\newblock Large language models are superpositions of all characters: Attaining arbitrary role-play via self-alignment.

\bibitem[\protect\citeauthoryear{Mou, Luo, Zhang, and Ye}{Mou et~al.}{2025}]{mou2025saroenhancingllmsafety}
Mou, Y., Y.~Luo, S.~Zhang, and W.~Ye (2025).
\newblock Saro: Enhancing llm safety through reasoning-based alignment.

\bibitem[\protect\citeauthoryear{Narvekar, Peng, Leonetti, Sinapov, Taylor, and Stone}{Narvekar et~al.}{2020}]{narvekar2020curriculumlearningreinforcementlearning}
Narvekar, S., B.~Peng, M.~Leonetti, J.~Sinapov, M.~E. Taylor, and P.~Stone (2020).
\newblock Curriculum learning for reinforcement learning domains: A framework and survey.

\bibitem[\protect\citeauthoryear{Ouyang, Wu, Jiang, Almeida, Wainwright, Mishkin, Zhang, Agarwal, Slama, Ray, et~al.}{Ouyang et~al.}{2022}]{ouyang2022training}
Ouyang, L., J.~Wu, X.~Jiang, D.~Almeida, C.~Wainwright, P.~Mishkin, C.~Zhang, S.~Agarwal, K.~Slama, A.~Ray, et~al. (2022).
\newblock Training language models to follow instructions with human feedback.
\newblock {\em Advances in Neural Information Processing Systems\/}~{\em 35}, 27730--27744.

\bibitem[\protect\citeauthoryear{Pang, Tang, Ye, Xiong, Zhang, Wang, and Chen}{Pang et~al.}{2024}]{pang2024selfalignmentlargelanguagemodels}
Pang, X., S.~Tang, R.~Ye, Y.~Xiong, B.~Zhang, Y.~Wang, and S.~Chen (2024).
\newblock Self-alignment of large language models via monopolylogue-based social scene simulation.

\bibitem[\protect\citeauthoryear{Parashar, Gui, Li, Ling, Vemuri, Olson, Li, Zhang, Caverlee, Kalathil, and Ji}{Parashar et~al.}{2025}]{parashar2025curriculumreinforcementlearningeasy}
Parashar, S., S.~Gui, X.~Li, H.~Ling, S.~Vemuri, B.~Olson, E.~Li, Y.~Zhang, J.~Caverlee, D.~Kalathil, and S.~Ji (2025).
\newblock Curriculum reinforcement learning from easy to hard tasks improves llm reasoning.

\bibitem[\protect\citeauthoryear{Patel}{Patel}{2024}]{patel2024understandingLLMs}
Patel, R. (2024).
\newblock Understanding llms from scratch using middle school math.

\bibitem[\protect\citeauthoryear{Radford, Narasimhan, Salimans, and Sutskever}{Radford et~al.}{2018}]{radford2018improving}
Radford, A., K.~Narasimhan, T.~Salimans, and I.~Sutskever (2018).
\newblock Improving language understanding by generative pre-training.
\newblock In {\em OpenAI}.

\bibitem[\protect\citeauthoryear{Radford, Wu, Child, Luan, Amodei, and Sutskever}{Radford et~al.}{2019}]{radford2019language}
Radford, A., J.~Wu, R.~Child, D.~Luan, D.~Amodei, and I.~Sutskever (2019).
\newblock Language models are unsupervised multitask learners.
\newblock {\em OpenAI blog\/}~{\em 1\/}(8), 9.

\bibitem[\protect\citeauthoryear{Rafailov, Sharma, Mitchell, Ermon, Manning, and Finn}{Rafailov et~al.}{2023}]{rafailov2023direct}
Rafailov, R., A.~Sharma, E.~Mitchell, S.~Ermon, C.~D. Manning, and C.~Finn (2023).
\newblock Direct preference optimization: Your language model is secretly a reward model.
\newblock {\em arXiv preprint arXiv:2305.18290\/}.

\bibitem[\protect\citeauthoryear{Rohde and Plaut}{Rohde and Plaut}{1999}]{rohde1999}
Rohde, D.~L. and D.~C. Plaut (1999).
\newblock Language acquisition in the absence of explicit negative evidence: How important is starting small?
\newblock In {\em Proceedings of the 21st Annual Conference of the Cognitive Science Society}, pp.\  618--623.

\bibitem[\protect\citeauthoryear{Roux, Bellemare, Lebensold, Bergeron, Greaves, Fréchette, Pelletier, Thibodeau-Laufer, Toth, and Work}{Roux et~al.}{2025}]{roux2025taperedoffpolicyreinforcestable}
Roux, N.~L., M.~G. Bellemare, J.~Lebensold, A.~Bergeron, J.~Greaves, A.~Fréchette, C.~Pelletier, E.~Thibodeau-Laufer, S.~Toth, and S.~Work (2025).
\newblock Tapered off-policy reinforce: Stable and efficient reinforcement learning for llms.

\bibitem[\protect\citeauthoryear{Sanger}{Sanger}{1994}]{sanger1994}
Sanger, T.~D. (1994).
\newblock Neural network learning control of robot manipulators using gradual learning.
\newblock In {\em Proceedings of 1994 IEEE International Conference on Robotics and Automation}, pp.\  1426--1431. IEEE.

\bibitem[\protect\citeauthoryear{Schulman, Levine, Abbeel, Jordan, and Moritz}{Schulman et~al.}{2015}]{schulman2015trust}
Schulman, J., S.~Levine, P.~Abbeel, M.~Jordan, and P.~Moritz (2015).
\newblock Trust region policy optimization.
\newblock {\em arXiv preprint arXiv:1502.05477\/}.

\bibitem[\protect\citeauthoryear{Schulman, Moritz, Levine, Jordan, and Abbeel}{Schulman et~al.}{2016}]{schulman2016high}
Schulman, J., P.~Moritz, S.~Levine, M.~I. Jordan, and P.~Abbeel (2016).
\newblock High-dimensional continuous control using generalized advantage estimation.
\newblock In {\em International conference on learning representations (ICLR)}.

\bibitem[\protect\citeauthoryear{Schulman, Wolski, Dhariwal, Radford, and Klimov}{Schulman et~al.}{2017}]{schulman2017proximal}
Schulman, J., F.~Wolski, P.~Dhariwal, A.~Radford, and O.~Klimov (2017).
\newblock Proximal policy optimization algorithms.
\newblock {\em arXiv preprint arXiv:1707.06347\/}.

\bibitem[\protect\citeauthoryear{{Sea AI Lab and National University of Singapore and Singapore Management University}}{{Sea AI Lab and National University of Singapore and Singapore Management University}}{2025}]{seaailab2025drgrpo}
{Sea AI Lab and National University of Singapore and Singapore Management University} (2025).
\newblock Dr. grpo: A bias-free reinforcement learning method that enhances math reasoning accuracy in large language models without inflating responses.
\newblock Research Announcement.

\bibitem[\protect\citeauthoryear{Shannon}{Shannon}{1948}]{shannon1948mathematical}
Shannon, C.~E. (1948).
\newblock A mathematical theory of communication.
\newblock {\em The Bell System Technical Journal\/}~{\em 27\/}(3), 379--423.

\bibitem[\protect\citeauthoryear{Shao, Wang, Zhu, Xu, Song, Li, Wu, and Guo}{Shao et~al.}{2024}]{shao2024deepseekmath}
Shao, Z., P.~Wang, Q.~Zhu, R.~Xu, J.~Song, Y.~Li, Y.~Wu, and D.~Guo (2024).
\newblock Deepseekmath: Pushing the limits of mathematical reasoning in open language models.
\newblock {\em arXiv preprint arXiv:2402.03300\/}.

\bibitem[\protect\citeauthoryear{Sharma, Keh, Mitchell, Finn, Arora, and Kollar}{Sharma et~al.}{2024}]{sharma2024critical}
Sharma, A., S.~Keh, E.~Mitchell, C.~Finn, K.~Arora, and T.~Kollar (2024).
\newblock A critical evaluation of {AI} feedback for aligning large language models.

\bibitem[\protect\citeauthoryear{Soviany, Ionescu, Rota, and Sebe}{Soviany et~al.}{2022}]{soviany2022curriculumlearningsurvey}
Soviany, P., R.~T. Ionescu, P.~Rota, and N.~Sebe (2022).
\newblock Curriculum learning: A survey.

\bibitem[\protect\citeauthoryear{Stiennon, Ouyang, Wu, Ziegler, Lowe, Voss, Radford, Amodei, and Christiano}{Stiennon et~al.}{2022}]{stiennon2022learningsummarizehumanfeedback}
Stiennon, N., L.~Ouyang, J.~Wu, D.~M. Ziegler, R.~Lowe, C.~Voss, A.~Radford, D.~Amodei, and P.~Christiano (2022).
\newblock Learning to summarize from human feedback.

\bibitem[\protect\citeauthoryear{Sutton}{Sutton}{1988}]{sutton1988learning}
Sutton, R.~S. (1988).
\newblock Learning to predict by the methods of temporal differences.
\newblock {\em Machine Learning\/}~{\em 3\/}(1), 9--44.

\bibitem[\protect\citeauthoryear{Thoppilan, De~Freitas, Hall, Shazeer, Kulshreshtha, Cheng, Jin, Bos, Baker, Du, Li, Zheng, Geng, Chen, Yin, Jindal, Vu, Li, Arya, Petrov, Bosma, Zhou, Chang, Grangier, Williams, Clark, Doherty, Tolpin, Chang, Chau, Zheng, Chi, Zhou, Dai, Chen, Cui, Le, Wu, Huang, Roberts, Zoph, and Wei}{Thoppilan et~al.}{2022}]{thoppilan2022lamda}
Thoppilan, R., D.~De~Freitas, J.~Hall, N.~Shazeer, A.~Kulshreshtha, H.-T. Cheng, A.~Jin, T.~Bos, L.~Baker, Y.~Du, Y.~Li, H.~S. Zheng, S.-y. Geng, K.~Chen, P.~Yin, M.~Jindal, T.~Vu, L.~J. Li, S.~Arya, S.~Petrov, M.~Bosma, Z.~Zhou, C.-C. Chang, D.~Grangier, P.~Williams, J.~H. Clark, R.~Doherty, A.~A. Tolpin, T.-H.~J. Chang, E.~H. Chau, S.~Zheng, E.~H. Chi, D.~Zhou, A.~M. Dai, Z.~Chen, C.~Cui, Q.~V. Le, Y.~Wu, A.~H.~H. Huang, A.~Roberts, B.~Zoph, and J.~Wei (2022).
\newblock {LaMDA}: Language models for dialog applications.
\newblock {\em arXiv preprint arXiv:2201.08239\/}.

\bibitem[\protect\citeauthoryear{Touvron, Lavril, Izacard, Martinet, Lachaux, Lacroix, Rozière, Goyal, Hambro, Azhar, Rodriguez, Joulin, Grave, and Lample}{Touvron et~al.}{2023}]{touvron2023llama}
Touvron, H., T.~Lavril, G.~Izacard, X.~Martinet, M.-A. Lachaux, T.~Lacroix, B.~Rozière, N.~Goyal, E.~Hambro, F.~Azhar, A.~Rodriguez, A.~Joulin, E.~Grave, and G.~Lample (2023).
\newblock Llama: Open and efficient foundation language models.
\newblock {\em arXiv preprint arXiv:2302.13971\/}.

\bibitem[\protect\citeauthoryear{Touvron, Martin, Stone, Albert, Almahairi, Babaei, Bashlykov, Batra, Bhargava, Bhosale, Bikel, Blecher, Ferrer, Chen, Cucurull, Esiobu, Fernandes, Fu, Fu, Fuller, Gao, Goswami, Goyal, Hartshorn, Hosseini, Hou, Inan, Kardas, Kerkez, Khabsa, Kloumann, Korenev, Koura, Lachaux, Lavril, Lee, Liskovich, Lu, Mao, Martinet, Mihaylov, Mishra, Molybog, Nie, Poulton, Reizenstein, Rungta, Saladi, Schelten, Silva, Smith, Subramanian, Tan, Tang, Taylor, Williams, Kuan, Xu, Yan, Zarov, Zhang, Fan, Kambadur, Narang, Rodriguez, Stojnic, Edunov, and Scialom}{Touvron et~al.}{2023}]{touvron2023llama2openfoundation}
Touvron, H., L.~Martin, K.~Stone, P.~Albert, A.~Almahairi, Y.~Babaei, N.~Bashlykov, S.~Batra, P.~Bhargava, S.~Bhosale, D.~Bikel, L.~Blecher, C.~C. Ferrer, M.~Chen, G.~Cucurull, D.~Esiobu, J.~Fernandes, J.~Fu, W.~Fu, B.~Fuller, C.~Gao, V.~Goswami, N.~Goyal, A.~Hartshorn, S.~Hosseini, R.~Hou, H.~Inan, M.~Kardas, V.~Kerkez, M.~Khabsa, I.~Kloumann, A.~Korenev, P.~S. Koura, M.-A. Lachaux, T.~Lavril, J.~Lee, D.~Liskovich, Y.~Lu, Y.~Mao, X.~Martinet, T.~Mihaylov, P.~Mishra, I.~Molybog, Y.~Nie, A.~Poulton, J.~Reizenstein, R.~Rungta, K.~Saladi, A.~Schelten, R.~Silva, E.~M. Smith, R.~Subramanian, X.~E. Tan, B.~Tang, R.~Taylor, A.~Williams, J.~X. Kuan, P.~Xu, Z.~Yan, I.~Zarov, Y.~Zhang, A.~Fan, M.~Kambadur, S.~Narang, A.~Rodriguez, R.~Stojnic, S.~Edunov, and T.~Scialom (2023).
\newblock Llama 2: Open foundation and fine-tuned chat models.

\bibitem[\protect\citeauthoryear{Uesato, Kushman, Kumar, Song, Siegel, Wang, Creswell, Irving, and Higgins}{Uesato et~al.}{2022}]{uesato2022solvingmathwordproblems}
Uesato, J., N.~Kushman, R.~Kumar, F.~Song, N.~Siegel, L.~Wang, A.~Creswell, G.~Irving, and I.~Higgins (2022).
\newblock Solving math word problems with process- and outcome-based feedback.

\bibitem[\protect\citeauthoryear{von Neumann}{von Neumann}{1951}]{von1951various}
von Neumann, J. (1951).
\newblock Various techniques used in connection with random digits.
\newblock In {\em Monte Carlo method}, Volume~12, pp.\  36--38. US Government Printing Office.

\bibitem[\protect\citeauthoryear{Wang, Hao, Dong, Zhang, Bao, Yang, and Wu}{Wang et~al.}{2024}]{wang2024offlinereinforcementlearningllm}
Wang, H., S.~Hao, H.~Dong, S.~Zhang, Y.~Bao, Z.~Yang, and Y.~Wu (2024).
\newblock Offline reinforcement learning for llm multi-step reasoning.

\bibitem[\protect\citeauthoryear{Wang, Chen, and Zhu}{Wang et~al.}{2021}]{wang2021surveycurriculumlearning}
Wang, X., Y.~Chen, and W.~Zhu (2021).
\newblock A survey on curriculum learning.

\bibitem[\protect\citeauthoryear{Watkins}{Watkins}{1989}]{watkins1989learning}
Watkins, C. J. C.~H. (1989).
\newblock {\em Learning from delayed rewards}.
\newblock Ph.\ D. thesis, University of Cambridge.

\bibitem[\protect\citeauthoryear{Williams}{Williams}{2024}]{williams2024morlaif}
Williams, M. (2024).
\newblock Multi-objective reinforcement learning from {AI} feedback.

\bibitem[\protect\citeauthoryear{Williams}{Williams}{1992}]{williams1992simple}
Williams, R.~J. (1992).
\newblock Simple statistical gradient-following algorithms for connectionist reinforcement learning.
\newblock {\em Machine learning\/}~{\em 8\/}(3), 229--256.

\bibitem[\protect\citeauthoryear{Wu, Sun, Yuan, Ji, Yang, and Gu}{Wu et~al.}{2024}]{wu2024selfplaypreferenceoptimizationlanguage}
Wu, Y., Z.~Sun, H.~Yuan, K.~Ji, Y.~Yang, and Q.~Gu (2024).
\newblock Self-play preference optimization for language model alignment.

\bibitem[\protect\citeauthoryear{Yuan et~al.}{Yuan et~al.}{2023}]{yuan2023scaling}
Yuan, A. et~al. (2023).
\newblock Scaling relationship on learning mathematical reasoning with large language models.
\newblock {\em arXiv preprint arXiv:2401.08226\/}.

\bibitem[\protect\citeauthoryear{Zelikman, Wu, Mu, and Goodman}{Zelikman et~al.}{2022}]{zelikman2022star}
Zelikman, E., Y.~Wu, J.~Mu, and N.~D. Goodman (2022).
\newblock Star: Self-taught reasoner.
\newblock {\em arXiv preprint arXiv:2203.14465\/}.

\bibitem[\protect\citeauthoryear{Ziegler, Stiennon, Wu, Brown, Radford, Amodei, Christiano, and Irving}{Ziegler et~al.}{2020}]{ziegler2020finetuninglanguagemodelshuman}
Ziegler, D.~M., N.~Stiennon, J.~Wu, T.~B. Brown, A.~Radford, D.~Amodei, P.~Christiano, and G.~Irving (2020).
\newblock Fine-tuning language models from human preferences.

\end{thebibliography}

\section{Appendix}

\subsection{Negative Log-likelihood (NLL)}\label{sec:nll}
Imagine that $p(x)$ is a probability distribution function. Here, $x$ can be anything, it could be a number, or a vector, or some other irregular-dimensional object. If $p$ is continuous, what this means is that $\delta\cdot p(x)$ is the probability that a random sample drawn from $p$ will be in the $\delta$ neighborhood of $x$ for vanishingly small $\delta$, and our sums below will become integration. Now let's write down the probability of a sample $\boldsymbol{x} = (x_1,x_2,x_3,\ldots,x_n)$ with $n$ observations. We will make the simplifying assumption that each sample drawn from the probability distribution is independent of another:
\begin{align*}
    P(\boldsymbol{x} | p) = \prod_{i=1}^n p(x_i)
\end{align*}
Now, the way real world is set up is that we often don't know the true distribution $p$. What we do have are the observations $\boldsymbol{x}$, and what we are really trying to do is generalize the nature of the world in some way that we can make useful predictions about it by learning from these observations $\boldsymbol{x}$. If we could figure out $p$, we'd be set. So in summary:
\begin{align*}
    \text{What we have: }& \boldsymbol{x}\text{, the observations from real world}\\
    \text{What we want: }& p \text{, the underlying probability distribution that generated the observations}
\end{align*}
Since the real world is infinitely complex, this is a hopeless pursuit. But, can we recover perhaps a simpler model of the world that may be useful in making some relevant predictions? Let's call this distribution $q$, and basically $q$ is something we will iterate on and try to get as close to $p$ as possible. If $q$ were the real state of the world, the probability of $\boldsymbol{x}$ under $q$ would be:
\begin{align*}
    P(\boldsymbol{x} | q) = \prod_{i=1}^n q(x_i)
\end{align*}
We could now iterate over $q$ to maximize this probability and we would be aligning our assumed model of the world to what we have really observed and hopefully $q$ will be a decent approximation of $p$. In this case, since we are iterating over $q$ and $\boldsymbol{x}$ is fixed, we could also interpret this probability as ``Likelihood of $q$ being the true state of the world given that $\boldsymbol{x}$ is what we have observed in reality''. This allows us to see world as a fixed reference point and change $q$ to align with reality. This is often written as 
\begin{align*}
 \mathcal{L}(q | \boldsymbol{x}) = P(\boldsymbol{x} | q) = \prod_{i=1}^n q(x_i)
\end{align*}
and the math doesn't change at all, this is just a matter of interpretation. One practical consideration is that these probabilities can be small, so multiplying them often leads to extremely small numbers (e.g. if an LLM vocabulary is 32k then probability of each token can be $\frac{1}{32000} = 3.13\times10^{-5}$ and if you have a sentence with 10 tokens the multiplying these 10 times puts you around $10^{-50}$.). To make computations manageable we can take a log, which also has the nice property that our product will turn into a sum (e.g. $\log{\left(\frac{1}{32000}\right)} = -10.37$ on the other hand, and much easier to manage - base here is $e$). Since log is a strictly monotone increasing function, we can maximize the log of probability instead of probability and it would amount to the same thing. Now, since probability is always less than 1, the log values will always be negative. What if we multiplied everything with a negative symbol to turn all the log values into positive numbers, with a wonderful side-effect that the maximization problem now turns into a minimization problem. So what we are really trying to do while finding a good $q$ is minimize the negative log-likelihood (NLL):
\begin{align*}
    \text{Minimize: NLL}(q|\boldsymbol{x}) =  -log\left(\mathcal{L}(q | \boldsymbol{x}) \right) = -\sum_{i=1}^n log\left(q(x_i)\right)
\end{align*}
The goal is to get the NLL of $q$ down as close as possible to the NLL of true distribution, $p$. We must keep the specification of $q$ sufficiently general, otherwise we risk overfitting to the observed sample while driving down the NLL. Likelihood was a concept introduced by \citet{fisher1922mathematical}.
\begin{quote}
\centering
Low NLL $\simeq$ High Likelihood $\simeq$ High probability    
\end{quote}

\subsection{Entropy}\label{sec:entropy}
Entropy is sort of how much ``randomness" is in a system. So for example, with the distribution $p$ that the world actually follows (we don't know $p$, of course, but we assume it exists), we wrote down the probability of sample $\boldsymbol{x}$ earlier. What's a reasonable definition of randomness when it comes to a probability distribution? What we are trying to capture is how deterministic is the realized value of a random variable. If the realized value is highly deterministic, let's say the random variable always take a fixed value of 1, then entropy should be zero. In this case the entire probability mass is centered on a single number. On the other hand, if it can take two values, say 0 and 1 with equal probability, then there is more entropy, but it is still contained. Each of 0 and 1 has a 0.5 probability of occurring. If we rolled a dice, then each number would have a probability of $1/6$ occurring, and a dice has more randomness. If we have many values, each one with a low probability, then the entropy is high. The lower this probability, the higher the entropy. To define entropy we could just take negative log likelihood of this probability. So in the first case that would be $-\log(1)=0$, in the second case it would be $-\log(1/2)= 0.69$ and in the case of the dice it would be $-\log(1/6)=1.79$. So we are getting somewhere. We have a way of defining entropy at least for these simple cases.

But not all distributions are uniform in the sense that each value a random variable can take does not occur with same probability. What if we have a loaded dice where probability of some numbers was higher than others? So what we want to do in the general case is take expectation of this NLL over all the values the random variable can take. 
\begin{align*}
    \text{Entropy} &= \text{Expected value of} -log(x) \text{ over all }x \\
        &= -\sum_x p(x)log(p(x))
\end{align*}
Where the sum is over all possible $x$. So, entropy is simply the expected value of NLL of the true distribution of the system. If the entropy is close to zero (and therefore the NLL is), you know $p$ fairly deterministically. In practice, what this means is that if you draw a sample and don't see much randomness (e.g. the sun has been rising to the east for several millennia) you can be reasonably sure of the underlying distribution $p$ (i.e. the sun should rise to the east tomorrow). On the other hand if the likelihood is low, you don't really have confidence in your $p$.

The concept of entropy originated in thermodynamics, where it was first introduced by \citet{clausius1865} as a measure of energy dispersal and process irreversibility. This macroscopic view was later given a microscopic, statistical foundation by \citet{boltzmann1877}, who related entropy to the number of possible microstates of a system. The concept found a powerful new application in the 20th century when \citet{shannon1948mathematical} independently developed a parallel formulation for information entropy, quantifying uncertainty and information content in communication.

\begin{quote}
\centering
Low NLL $\simeq$ High Likelihood $\simeq$ High probability $\simeq$ High certainty $\simeq$ Low randomness
\end{quote}

\subsection{Cross entropy}\label{sec:crossentropy}
In practice, we don't really know $p$, the true distribution. What we are trying to do is calculate the NLL of $q$ and then minimize that to get as close to $p$ as possible. And in doing so, what we are doing is calculating the NLL using samples from the real world. These samples are drawn from the true distribution of the real world - $p$. As such, if we were to calculate the expected value of the NLL that we are trying to optimize, that would be cross entropy, and would come out to be something like this:
 \begin{align*}
 \text{Cross entropy} &= \text{Expected value of NLL$(q|x)$ when real distribution is $p$}\\ &= - \sum_{x} p(x) log(q(x))
 \end{align*}
 So basically, what entropy is measuring is expected NLL of true distribution when we draw a sample. What cross entropy is measuring is expected NLL of our approximate distribution when we draw a sample. Cross entropy can never be less than entropy, a fact that follows from our comment about NLL above but which we will not prove here. What you really want is to minimize the cross entropy value to get as close as possible to the entropy so as to get $q$ close to the true distribution $p$.

 But, you cannot really calculate cross entropy since the true distribution $p$ isn't known. So you can only approximate it with the observed samples. For example, if you got a sample $x_1, x_2,\ldots, x_n$ then you can assume that they have probability $p(x_i)=1$ under true distribution (since that is what we observed) and calculate cross entropy which is, of course the same thing as NLL of $q$ (since we got to cross entropy by taking expectation of NLL in the first place). 
\begin{align*}
    \text{Minimize: Cross entropy loss = } -\sum_{i=1}^n log\left(q(x_i)\right) = -log\left(\mathcal{L}(q | \boldsymbol{x}) \right) = \text{NLL}(q|\boldsymbol{x})
\end{align*}
 
 As such, when we use cross entropy loss what we are really using is the NLL of our model $q$, and as such the term negative log-likelihood more accurately describes the cross entropy loss. The concept of cross-entropy originates from the foundational principles of information entropy introduced by \citet{shannon1948mathematical} and was more formally derived from the related work on relative entropy by \citet{kullback1951information}. Its modern application as a pivotal loss function in machine learning is anchored by work on neural network training and model distillation, as exemplified by \citet{hinton2015distilling}.

\subsection{KL Divergence}\label{sec:kldivergence}
Now that we understand entropy and cross entropy, KL divergence, due to  \citet{kullback1951information}, is fairly straightforward to understand. It is simply the difference between cross entropy and entropy:
\begin{align*}
    \text{KL Divergence $(p \| q)$} &= \text{Cross entropy - entropy}\\
    D_{KL}(p\|q)&= \left(- \sum_{x} p(x) log(q(x))\right) - \left(- \sum_{x} p(x) log(p(x))\right)\\
    D_{KL}(p\|q) &= \sum_x p(x) log\left(\frac{p(x)}{q(x)}\right)
\end{align*}
What KL divergence measures is how much expected NLL is higher under the approximate distribution $q$ compared to the true distribution $p$. In essence, KL divergence is a measure of how much extra we are paying (in terms of NLL) by assuming the world behaves like $q$ when the true distribution is $p$. Since entropy of the system is fixed, minimizing cross entropy is the same thing as minimizing KL divergence. 

KL divergence is a measure of how far a distribution $q$ is from the real distribution $p$. If we look at the terms we can see that KL divergence will heavily penalize situations where the distribution $q$ assigns a very low probability to common events on the real distribution $p$ (since $q(x)$ will go to zero, making the term go to infinity). This can cause $q$ to distribute probability across the possibilities and explore more. KL divergence does not penalize nearly as much if $q$ also assigns some probability to parts of distribution that are highly unlikely in the true distribution $p$ since the term simply goes to zero.

In certain situations, you find yourself sampling from distribution $q$ rather than $p$ and you still would like to have a measure of distance of the two distributions $p$ and $q$. In such cases, you use KL divergence with the roles of $p$ and $q$ flipped such that 
\begin{align*}
    D_{KL}(q\|p) &= \sum_x q(x) log\left(\frac{q(x)}{p(x)}\right)
\end{align*}
This is sometimes also called reverse KL divergence (i.e. reverse of $D_{KL}(p\|q)$). The effect of this is that the formulation penalizes $q$ heavily for assigning probability to events that are highly unlikely under $p$ (since $p(x)$ goes to zero). This has the opposite effect of making $q$ focus on known areas in $p$.

\subsection{Importance Sampling}\label{sec:importancesampling}
Consider the situation where you are interested in estimating the average height of a person at a basketball event. The event consists of about 5\% of pro basketball players and 95\% people from the general population. You are able to get the height of about 200 randomly selected pro-players at the event, from publicly available data about them (let's say $x_1,...x_{200}$). You are also able to get the height of about 100 randomly identified participants at the event (let's say $x_{201},...,x_{300}$). How do you calculate the average height? You know that taking a simple average of the 200 numbers is going to be wrong, since there are far fewer pro basketball players at the event than the general population, you know that pro basketball players are taller as a rule, and you know that only 5\% of the people at the event are pro players. Assigning those 100 pro-player samples a weight of 50\% (as would happen if you were to simply average all samples) would be wrong. What is the simplest thing you can do?

A reasonable thing to do would be to weight the samples based on how much of that is present in the full event population. We know the average height of pro-players is average of $x_1,...,x_{200}$ and the average height of everyone else is average of $x_{201},...,x_{300}$. A good estimate of the overall average would be to weight the average of the first sample at 5\% and the second one at 95\% and so:
\begin{align*}
    Estimate &= 0.05\left(\frac{x_1+x_2+...+x_{200}}{200}\right) + 0.95\left(\frac{x_{201}+x_{202}+...+x_{300}}{100}\right)
\end{align*}
Let's formalize this a little bit. Let $P$ be the original distribution of pro-players and non-players in the conference and since we have 5\% pro players: $P(pp)=0.05$ and $P(np)=0.95$. Let's also use $Q$ to denote the distribution which we used to sample. Since we have 200 pp and 100 np the probability distribution that $Q$ follows is $Q(pp)=\frac{2}{3}$ and $Q(np)=\frac{1}{3}$. Let's try to rewrite the above formula in terms of $P$ and $Q$:
\begin{align*}
    Estimate &= 0.05\left(\frac{x_1+x_2+...+x_{200}}{200}\right) + 0.95\left(\frac{x_{201}+x_{202}+...+x_{300}}{100}\right)\\
    &= \frac{0.05}{200}\sum_{i=1}^{200} x_i + \frac{0.95}{100} \sum_{i=201}^{300} x_i\\
    &= \frac{1}{300}\left[ \frac{0.05}{\frac{2}{3}}\sum_{i=1}^{200} x_i + \frac{0.95}{\frac{1}{3}} \sum_{i=201}^{300} x_i \right]\\
    &= \frac{1}{300}\left[ \frac{P(pp)}{Q(pp)}\sum_{i=1}^{200} x_i + \frac{P(np)}{Q(np)} \sum_{i=201}^{300} x_i \right]\\
    &= \frac{1}{300}\left[ \sum_{i=1}^{200} \frac{P(pp)}{Q(pp)} x_i +  \sum_{i=201}^{300} \frac{P(np)}{Q(np)} x_i \right]
\end{align*}
This is the core idea of importance sampling where you sample from a distribution $Q$ (either because it is more convenient or $P$ is hard to sample from) and you can get an estimate at population level by weighing the samples appropriately by the ratio of the two distributions. Let's generalize this further. Let's assume that we have $N$ observations total sampled from the population following the distribution $Q$ where $Q(x_i)$ is the probability of $x_i$ being sampled, and $P(x_i)$ is the true probability. Then we can write the above formula as:
\begin{align}\label{eq:importancesampling}
    Estimate &= \frac{1}{N}\left[ \sum_{i=1}^{N} \frac{P(x_i)}{Q(x_i)} x_i \right]
\end{align}
This is importance sampling. The concept of importance sampling as a variance reduction technique for Monte Carlo methods has its roots in the late 1940s and early 1950s, primarily driven by research in statistical physics. \citet{kahn1949random} laid the foundational groundwork for the method, demonstrating how to use a biased sampling distribution to more efficiently estimate quantities of interest in neutron transport problems. This idea was further formalized and popularized in \citet{kahn1951estimation}, which provided a more rigorous mathematical framework and showcased the significant efficiency improvements achievable.
\subsection{Critique system prompt for code style}\label{sec:critiquesysprompt}
You are a highly experienced and meticulous code style critic. Your role is to analyze a given code block and provide a detailed, objective critique of its stylistic qualities depending on the programming language of the code:

\textbf{Your Persona:}
\begin{itemize}
    \item \emph{Knowledgeable:} You have a deep understanding of code conventions, design principles, and best practices across many programming languages (e.g., Python's PEP 8, JavaScript's Airbnb style guide, etc.).
    \item \emph{Objective:} Your critique is based on established principles, not personal preference. You will cite the reasoning behind your suggestions.
    \item \emph{Helpful:} Your goal is to guide the user to write more readable, maintainable, and professional code. Your tone should be constructive and encouraging.
\end{itemize}

\textbf{Core Task:}
Analyze the user-provided code block and perform a critique based on the following criteria.

\textbf{Critique Criteria:}
\begin{enumerate}
    \item \emph{Readability \& Formatting:}
    \begin{itemize}
        \item Is the code well-formatted and easy to read?
        \item Does it use consistent indentation, spacing, and line breaks?
        \item Are complex lines of code broken down logically?
    \end{itemize}
    \item \emph{Naming Conventions:}
    \begin{itemize}
        \item Are variable, function, and class names descriptive and meaningful?
        \item Does the code consistently follow a single casing style (e.g., \texttt{snake\_case}, \texttt{camelCase}, \texttt{PascalCase})?
        \item Are acronyms and abbreviations handled consistently?
    \end{itemize}
    \item \emph{Comments \& Documentation:}
    \begin{itemize}
        \item Are comments used effectively to explain \emph{why} the code does something, not just \emph{what} it does?
        \item Are complex functions or classes accompanied by appropriate documentation (e.g., docstrings)?
        \item Are comments formatted consistently?
    \end{itemize}
    \item \emph{Modularity \& Structure:}
    \begin{itemize}
        \item Is the code broken down into logical, single-purpose functions or methods?
        \item Are functions or methods excessively long?
        \item Are code blocks well-organized and clearly separated?
    \end{itemize}
    \item \emph{Robustness \& Best Practices:}
    \begin{itemize}
        \item Does the code handle potential errors gracefully (e.g., no bare \texttt{except} blocks)?
        \item Does it avoid hardcoding ``magic numbers'' or string literals?
        \item Is resource management handled properly (e.g., closing files, releasing connections)?
    \end{itemize}
\end{enumerate}

\textbf{Output Format:}
Your response must be structured clearly:
\begin{itemize}
    \item \emph{Summary:} Start with a brief, high-level summary of the code's overall style.
    \item \emph{Detailed Critique:} For each of the five criteria listed above, provide a specific critique.
    \begin{itemize}
        \item Clearly state which criterion you are addressing.
        \item For each point of critique, quote a specific line or block of code from the user's input to illustrate the issue.
        \item Explain the problem with that specific code snippet.
        \item Provide a concrete, improved example showing how to fix the issue.
    \end{itemize}
    \item \emph{Overall Recommendation:} Conclude with a final paragraph that summarizes your key recommendations and provides a final word of encouragement to the user.
\end{itemize}

\textbf{Constraints:}
\begin{itemize}
    \item Your output should be constructive, not condescending.
    \item Provide code examples for fixes.
    \item If the code is exemplary, praise its strengths and highlight which conventions it follows well.
\end{itemize}

\subsection{Proof: Variance of Weighted vs. Unweighted Average}
\begin{lemma}\label{lemma:confweighted}
If we have a sample of $\rho$ observations, $\tau_i, i \in 1,...\rho$, each with standard deviation $\sigma_i$ and we have $\bar{\varphi}_i$ inversely proportional to variance of the corresponding $\tau_i$, then:
\begin{align*}
    Var\left(\frac{1}{\rho} \sum_{i=1}^\rho \tau_i \right) &\ge  Var\left( \frac{\sum_{i=1}^\rho \tau_i \bar{\varphi}_i}{\sum_{i=1}^\rho \bar{\varphi}_i} \right) 
\end{align*}
\end{lemma}
\begin{proof} We have:
    \begin{align*}
        Var\left(\frac{1}{\rho} \sum_{i=1}^\rho \tau_i \right) &= \frac{1}{\rho^2} \sum_{i=1}^\rho \sigma_i^2
    \end{align*}
    
    And:
    \begin{align*}
        Var\left( \frac{\sum_{i=1}^\rho \tau_i \bar{\varphi}_i}{\sum_{i=1}^\rho \bar{\varphi}_i} \right) &= Var\left(\frac{\sum_{i=1}^\rho \tau_i (k/ \sigma_i^2)}{\sum_{j=1}^\rho (k/ \sigma_j^2)} \right)&\because \bar{\varphi}_i = \frac{k}{\sigma_i^2} \text{for some constant $k$} \\
        &= \frac{\sum_{i=1}^\rho \sigma_i^2 \left(\frac{1}{\sigma_i^2}\right)^2}{\left(\sum_{j=1}^\rho \frac{1}{\sigma_j^2}\right)^2}\\
        &= \frac{1}{\sum_{i=1}^\rho \frac{1}{\sigma_i^2}}
    \end{align*}
The Cauchy-Schwarz inequality, states $(\sum u_i^2)(\sum v_i^2) \ge (\sum u_i v_i)^2$. Let $u_i = \sigma_i$ and $v_i = 1/\sigma_i$.
\begin{align*}
    \left(\sum_{i=1}^\rho \sigma_i^2\right) \left(\sum_{i=1}^\rho \left(\frac{1}{\sigma_i}\right)^2\right) &\ge \left(\sum_{i=1}^\rho \sigma_i \cdot \frac{1}{\sigma_i}\right)^2 \\
    \left(\sum_{i=1}^\rho \sigma_i^2\right) \left(\sum_{i=1}^\rho \frac{1}{\sigma_i^2}\right) &\ge \left(\sum_{i=1}^\rho 1\right)^2 \\
    \left(\sum_{i=1}^\rho \sigma_i^2\right) \left(\sum_{i=1}^\rho \frac{1}{\sigma_i^2}\right) &\ge \rho^2\\
     \frac{1}{\rho^2} \sum_{i=1}^\rho \sigma_i^2 &\ge  \frac{1}{\sum_{i=1}^\rho \frac{1}{\sigma_i^2}}\\
      Var\left(\frac{1}{\rho} \sum_{i=1}^\rho \tau_i \right) &\ge  Var\left( \frac{\sum_{i=1}^\rho \tau_i \bar{\varphi}_i}{\sum_{i=1}^\rho \bar{\varphi}_i} \right) 
\end{align*}
\end{proof}

\end{document}